\documentclass[sigconf,nonacm]{acmart}
\usepackage{bm}
\usepackage{amsfonts}
\usepackage{graphicx}
\usepackage{textcomp}
\usepackage{xcolor}
\usepackage{url}
\usepackage[noend]{algpseudocode}
\usepackage{algorithm}
\usepackage{array}
\usepackage{booktabs}
\usepackage{amsthm}
\usepackage{multirow}
\usepackage{graphicx}
\usepackage{caption}
\usepackage{subcaption}
\usepackage{booktabs}
\usepackage{amsmath}
\usepackage{ulem}
\usepackage{pgf}
\usepackage{pgffor}
\usepackage{hyperref}
\usepackage{xspace}
\usepackage{setspace}
\usepackage[skip=0pt]{caption} % Adjust caption spacing
\usepackage{etoolbox}
\definecolor{green(html/cssgreen)}{rgb}{0.0, 0.5, 0.0}
\newcommand{\codecmt}[1]{{\it \small \color{green(html/cssgreen)} $\triangleright$ #1}}
\usepackage{booktabs} % 用于绘制高质量的横线 (\toprule, \midrule, \bottomrule)

\usepackage{amssymb}  % 用于绘制对勾 \checkmark 和 叉号 \cross
\usepackage{pifont}   % 另一种叉号的选择
\newcommand{\cmark}{\checkmark}
\newcommand{\xmark}{\ding{55}} % 需要 pifont 宏包，也可以直接用 \times

\usepackage{xcolor}
\usepackage{listings}
\usepackage[most]{tcolorbox}
\definecolor{myblue}{RGB}{0, 0, 200}
\definecolor{bgblue}{RGB}{245, 247, 255}
\definecolor{revisionpurple}{RGB}{0, 0, 0}

\usepackage{booktabs} % 提供 \toprule, \midrule, \bottomrule
\usepackage[table]{xcolor} % 提供表头底色功能

\usepackage{float}
\usepackage{enumitem}
\usepackage{lipsum}   % 示例文本生成

\setlist{nosep, leftmargin=*}
\newtheorem{theorem}{\textbf{Theorem}}
\newtheorem{lemma}{\textbf{Lemma}}

\newcommand{\ModelName}{\texttt{LUCID}\xspace}

\makeatletter
\patchcmd{\proof}
  {\topsep}
  {3pt \@plus 0.1ex \@minus 0.1ex}
  {}{}
\patchcmd{\endproof}
  {\topsep}
  {3pt \@plus 0.1ex \@minus 0.1ex}
  {}{}
\makeatother

\copyrightyear{2027}
\acmYear{2027}
\setcopyright{acmlicensed}
\acmConference[KDD '27]{Proceedings of the 33rd ACM SIGKDD Conference on Knowledge Discovery and Data Mining V.1}
{August 1--5, 2027}
{San Jose, CA, USA}
\begin{document}

    % \title{Effective and Explainable Unsupervised Community Detection with LLM-guided Rule Inductor}
    \title{Interpretable Unsupervised Community Detection with LLM-Symbolized Structured Processes}
    % Reasoning on Structure for Interpretable Unsupervised Community Detection
    % Explainable Unsupervised Community Detection with LLM-Induced Structural Decision Rules
    % Interpretable Unsupervised Community Detection with LLM-Symbolized Structured Processes
    \author{Aoting Zeng}
    \orcid{0009-0002-6590-2905}
    \affiliation{
    \institution{Shanghai Jiao Tong University} \city{Shanghai} \country{China} 
    %\school{School of Software Technology}   
    }
    \email{aotingzeng@sjtu.edu.cn} 
    
    \author{Kai Wang}
    \orcid{0000-0002-3123-2184}
    \affiliation{
    \institution{Shanghai Jiao Tong University} \city{Shanghai} \country{China} 
    }
    \email{w.kai@sjtu.edu.cn}
    
    \author{Jianwei Wang}
    \orcid{0009-0000-7887-4179}
    \affiliation{
    \institution{University of New South Wales} \city{Sydney} \country{Australia}  
    }
    \email{jianwei.wang1@unsw.edu.au}

    \author{Yuxiang Sun}
    \orcid{0009-0001-0848-7955}
    \affiliation{
    \institution{Shanghai Jiao Tong University} \city{Shanghai} \country{China} 
    }
    \email{scscsc04@sjtu.edu.cn}
    
    \author{Yizhang He}
    \orcid{0000-0002-4426-7503}
     \affiliation{
    \institution{Shanghai Jiao Tong University} \city{Shanghai} \country{China} 
    }
    \email{yizhang.he@sjtu.edu.cn}

    \author{Wenjie Zhang}
    \orcid{0000-0001-6572-2600}
     \affiliation{
    \institution{University of New South Wales} \city{Sydney} \country{Australia} 
    }
    \email{wenjie.zhang@unsw.edu.au}
    \renewcommand{\shortauthors}{Aoting Zeng, Kai Wang, Jianwei Wang, Yuxiang Sun, Yizhang He, Wenjie Zhang}

\begin{abstract}
    Community detection is a fundamental task in graph analytics that aims to identify cohesive groups of entities with similar behaviors or interests. 
Classic objective-driven methods struggle with complex graph structures, while deep-learning approaches improve performance at the expense of interpretability and their reliance on labeled data and training. 
Large language models (LLMs), with strong reasoning and world knowledge, are promising for interpretable, label-free community detection. 
To leverage these strengths, we propose \ModelName, an \textbf{\underline{L}}LM-guided interpretable, training-free and \textbf{\underline{U}}nsupervised \textbf{\underline{C}}ommun\textbf{\underline{I}}ty \textbf{\underline{D}}etection method. Inspired by phase-transition kinetics in natural systems, where complex structures emerge through initialization, merging, refinement, and selection, \ModelName is designed as a four-stage pipeline. 
Within this pipeline, the LLM induce formal rules that translate implicit knowledge into explicit and interpretable logical structures. 
 Specifically, 
(1) the \textit{Local-View Community Initialization} stage encodes local graph structure with $k$-ego contexts and unsupervised node roles; 
(2) the \textit{Multi-factor Community Merge} stage uses LLM-induced rules to iteratively merge local communities;
(3) the \textit{Multi-grain Community Refinement} stage applies LLM-induced coarse-to-fine rules in parallel to reduce boundary noise; and 
(4) the \textit{Global-view Community Selection} stage identifies high-quality communities with topological compactness and boundary clarity. 
Extensive experiments on real-world datasets demonstrate that \ModelName, as an unsupervised approach, achieves state-of-the-art performance and consistently outperforms leading unsupervised and semi-supervised baselines. 
    \end{abstract}

% \ccsdesc[500]{Information systems → Data mining.}

    \maketitle
    \newcommand\kddavailabilityurl{https://doi.org/10.5281/zenodo.15531996}
\label{code}
\ifdefempty{\kddavailabilityurl}{}{
\begingroup\small\noindent\raggedright\textbf{KDD Availability Link:}\\
The source code of this paper is available at \url{https://anonymous.4open.science/r/KDD2027LUCID-0F70}.
\endgroup
}

    \section{Introduction}
    \label{sec.intro}
    
%% graph important and CD is important. 
Graphs serve as powerful tools for modeling complex relationships and have been widely applied in domains including financial systems \cite{xiang2023semi,ma2023fighting}, social networks \cite{backstrom2006group,wu2025gas,meng2024interpretable}, and biomedicine \cite{Krogan_Cagney,xu2025survey}. 
In graph analytics, community detection \cite{fortunato2010community,newman2004finding} is a fundamental task that identifies densely connected groups of entities sharing similar behaviors or interests, which provides essential insights for downstream analysis and strategic decision-making. 
% Uncovering these structures provides essential insights for downstream analysis and strategic decision-making. 

%% traditional methods not flexible to capture complex patterns. recent shifts to learning methods for  flexibility 
\textcolor{black}{Classic objective-based community detection methods \cite{newman2004finding,blondel2008fast,traag2019louvain} often struggle to capture complex community structures in real-world graphs due to limited flexibility.} 
More recently, a growing body of work \cite{procd,come,seal,clare,procom} has shifted toward deep-learning–based approaches, which leverage the capabilities of neural networks to capture complex structural patterns in graphs, as shown in Table~\ref{tab:comparison-ed}. 
For instance, ComE \cite{come} exploits node embeddings to encode community-related features, whereas SEAL \cite{seal} employs generative adversarial networks to infer community structures. CLARE \cite{clare} introduces a deep reinforcement learning framework with locator–rewriter mechanisms, and PROCOM \cite{procom} integrates pre-training with prompt-based strategies to improve the performance.

% The deep–learning–based approaches~\cite{come,seal,clare,procom} achieve state-of-the-art results on benchmark datasets, but they still face these fundamental limitations in practice: 
% (1) they rely heavily on extensive supervised training, requiring large amounts of labeled data that are costly to obtain and often unavailable in the real-world; 
\textcolor{revisionpurple}{Deep-learning-based community detection methods operate under both unsupervised and semi-supervised settings. Despite their strong performance, they share two practical limitations:}
\textcolor{revisionpurple}{(1) their decision-making processes are generally opaque, making it difficult to interpret why specific nodes are grouped together; and}
\textcolor{revisionpurple}{(2) they often require dataset-specific training or adaptation when applied to new graphs.}
\textcolor{revisionpurple}{In addition, semi-supervised methods such as SEAL, CLARE, and PROCOM face a further limitation: they rely on labeled communities for training or prompting, while such annotations are costly to obtain and often unavailable in real-world applications.}

\begin{table*}[htbp]
\centering
\caption{Comparisons of community detection methods.}
\label{tab:comparison-ed}
\begin{tabular}{>{\mdseries}l>{\mdseries}c>{\mdseries}c>{\mdseries}c>{\mdseries}c>{\mdseries}c}
\toprule
\textbf{Approach} & \textcolor{revisionpurple}{\textbf{Overlap}} & \textbf{Label Free?} & \textcolor{revisionpurple}{\textbf{Detection Mode}} & \textbf{Backbone} & \textbf{Methodology} \\
\midrule
Louvain \cite{traag2019louvain} & \textcolor{revisionpurple}{\xmark} & \cmark & \textcolor{revisionpurple}{Global} & Modularity Optimization & Classic objective-based \\
\textcolor{revisionpurple}{PRoCD \cite{procd}} & \textcolor{revisionpurple}{\xmark} & \textcolor{revisionpurple}{\cmark} & \textcolor{revisionpurple}{Global} & \textcolor{revisionpurple}{Pre-train \& Refine} & \textcolor{revisionpurple}{Deep-learning-based} \\
BigClam \cite{bigclam} & \textcolor{revisionpurple}{\cmark} & \cmark & \textcolor{revisionpurple}{Global} & Matrix Factorization & Classic objective-based \\
ComE \cite{come}& \textcolor{revisionpurple}{\cmark} & \cmark & \textcolor{revisionpurple}{Global} & Node Embeddings & Deep-learning-based \\
CommunityGAN \cite{communitygan}& \textcolor{revisionpurple}{\cmark} & \cmark & \textcolor{revisionpurple}{Global} & GAN + Motifs & Deep-learning-based \\
Bespoke \cite{bespoke}& \textcolor{revisionpurple}{\cmark} & \xmark & \textcolor{revisionpurple}{Local} & Structural Fingerprints & Classic objective-based \\
SEAL \cite{seal}& \textcolor{revisionpurple}{\cmark} & \xmark & \textcolor{revisionpurple}{Local} & GAN & Deep-learning-based \\
CLARE \cite{clare}& \textcolor{revisionpurple}{\cmark} & \xmark & \textcolor{revisionpurple}{Local} & DRL (Locator-Rewriter) & Deep-learning-based \\
PROCOM \cite{procom}& \textcolor{revisionpurple}{\cmark} & \xmark & \textcolor{revisionpurple}{Local} & Pre-train \& Prompt & Deep-learning-based \\
\midrule
\ModelName (ours) & \textcolor{revisionpurple}{\cmark} & \cmark & \textcolor{revisionpurple}{Global} & LLM-Symbolized Pipeline & Symbolic rule-based \\
\bottomrule
\end{tabular}
\vspace{-8pt}
\end{table*}

\noindent\textbf{Motivations}.
% The above discussion motivates us to explore alternative paradigms for community detection that are interpretable, efficient, and label-free. 
%The above discussion motivates us to seek alternative community detection paradigms that do not rely on labeled data, are interpretable, and training-free. 
Motivated by the above discussions, we aim to explore new community detection paradigms that can identify community structures without heavy supervision, while providing interpretable insights into why specific communities are formed.
% Interpretability means a transparent decision process from inputs to outputs.
% Here, interpretable means understanding how inputs lead to outputs.
% Interpretability is the ability to understand how inputs lead to outputs.
In recent years, large language models (LLMs) have demonstrated strong reasoning capabilities and the ability to transfer knowledge without task-specific training \cite{zhao2023survey,naveed2025comprehensive}, making them a natural candidate for alleviating the interpretability and label-dependency limitations of existing approaches. 
%In particular, LLMs are good at understanding complex relationships, reasoning over structural patterns, and transferring knowledge across domains without task-specific training, which aligns well with the goals of interpretable and label-free community detection. 
However, directly applying LLMs to predict communities is impractical. Approaches that input the full graph or iterate over nodes suffer from context constraints and low efficiency.
% \textcolor{purple}{However, directly applying LLMs to determine node memberships remains impractical for large-scale graphs due to limited context windows, high inference costs, and low efficiency. }
How to leverage LLMs' reasoning capabilities while alleviating these issues has become a key challenge in introducing them into the community detection task. 
% In recent years, Large Language Models (LLMs) have demonstrated strong reasoning capabilities and rich world knowledge, making them a promising candidate for alleviating the interpretability and label-dependency limitations of existing approaches. In particular, LLMs excel at understanding complex relationships, reasoning over structural patterns, and transferring knowledge across domains without task-specific supervision, which aligns well with the goals of interpretable and label-efficient community detection. However, directly applying LLMs to determine node memberships remains impractical for large-scale graphs due to limited context windows, high inference costs, and low efficiency. Moreover, such end-to-end LLM inference still offers limited transparency in how decisions are made. Consequently, a key challenge lies in how to leverage LLMs’ reasoning strengths while avoiding these practical bottlenecks when introducing them into community detection.

\noindent\textbf{Our approaches}.
Motivated by the challenges and to address the above limitations, we propose \ModelName, an \textbf{\underline{L}}LM-guided interpretable, training-free and \textbf{\underline{U}}nsupervised \textbf{\underline{C}}ommun\textbf{\underline{I}}ty \textbf{\underline{D}}etection method. 
% \textcolor{black}{The design of LUCID is inspired by phase-transition kinetics in natural systems \cite{barish2009information}, where complex structures emerge through sequential processes of initialization, merging, refinement, and selection.}
% Following this natural progression, \ModelName adopts a similar four-stage pipeline to progressively construct communities from local structures to global view.
\textcolor{revisionpurple}{As a result, \ModelName incurs no training overhead and requires no local GPU resources, facilitating deployment in resource-constrained settings.}
% \textcolor{revisionpurple}{Its training-free design avoids task-specific model training and reduces dependence on GPU resources, making \ModelName suitable for resource-constrained settings.}
\textcolor{revisionpurple}{The design of LUCID is inspired by phase-transition kinetics in natural systems \cite{barish2009information,Li2025Nucleation}, which provides a high-level evolutionary view of how complex structures emerge through successive stages.}
\textcolor{revisionpurple}{Following this view, \ModelName frames community formation as a four-phase process of initialization, merging, refinement, and selection, progressively constructing communities from local structures to a global view.}
\textcolor{black}{In particular, \ModelName leverages the LLM in the merge and refinement stages to make critical structural decisions. Rather than using the LLM as an end-to-end predictor, \ModelName employs it as an automatic rule inducer, which converts implicit model knowledge into explicit, interpretable rules. 
In this way, we combine the reasoning capabilities of the LLM with the structure of algorithmic graph processing. }

Specifically, (1) in the local-view community initialization stage, \ModelName transforms the graph into localized $k$-ego representations and introduces unsupervised node role labels to help the LLM understand local structural semantics. 
\textcolor{black}{(2) In the multi-factor community merge stage, the LLM first derives a multi-factor decision tree for evaluating structural compatibility between communities. 
This decision tree is then applied to guide merge decisions between neighboring communities. To improve both performance and scalability, we prioritize community backbones by sorting center nodes according to their ego-network density, ensuring high-coverage merges of core communities. We also select neighbors with high Jaccard similarity for each center node, restricting merge candidates to structurally relevant nodes while respecting context constraints.}
% \textcolor{purple}{(2)In the multi-factor community merge stage, the LLM derives a multi-factor decision tree for evaluating structural compatibility between communities. 
% To improve both performance and scalability, we employ order-aware center execution that prioritizes structurally rich regions for high-coverage merges, and budget-guided neighbor sampling that selects top-$K$ structurally relevant neighbors to reduce noise while respecting context constraints. 
(3) In the multi-grain community refinement stage, \ModelName guides the LLM to induce coarse-to-fine rules for refining node memberships in the community and removing boundary noise. 
These rules are applied progressively to refine nodes within each candidate community, and executed in parallel to enable concurrent processing across communities and batch refinement within each community.
(4) In the global-view community selection stage, we introduce a structural quality metric that evaluates topological compactness and boundary clarity to select high-quality communities in an interpretable manner.
% we propose a structural quality metric evaluating topological compactness and boundary clarity to select high-quality communities.
% In the multi-factor community merge stage, the LLM derives a multi-factor decision tree to evaluate structural compatibility between communities. To enhance scalability and merge quality, we employ order-aware center execution, which prioritizes structurally rich regions, and budget-guided neighbor sampling, which selects the top-$K$ relevant neighbors to reduce noise while respecting context constraints. In the multi-grain community refinement stage, LUCID guides the LLM to generate coarse-to-fine rules and executes them in parallel, leveraging the independence of node-level decisions to enable simultaneous refinement of multiple communities and batch processing within each community, significantly improving efficiency. Finally, in the global-view community selection stage, we introduce a structural quality metric that evaluates topological compactness and boundary clarity to select high-quality communities in an interpretable manner.

\noindent\textbf{Contributions}. Here, we summarize the \textcolor{revisionpurple}{principal} contributions.
% The contributions of this paper are summarized as follows. 
% \vspace{4pt}
\begin{itemize}
    \item We propose \ModelName, a symbolic rule–based community detection framework that is interpretable, training-free and label-free.

    \item \textcolor{revisionpurple}{We formulate community detection as a four-stage process of initialization, merging, refinement, and selection, together with stage-specific strategies for scalable and robust graph processing.}

    \item We propose a new way to leverage \textcolor{revisionpurple}{the LLM} as rule inducers, distilling implicit model knowledge into explicit, interpretable logic to drive decision-making within the four-stage algorithmic pipeline.

    \item Extensive experiments on diverse real-world datasets \textcolor{revisionpurple}{demonstrate state-of-the-art performance under unsupervised settings, while comprehensive analyses validate the interpretability and effectiveness of the proposed components.}
\end{itemize}
    \vspace{-3pt}
    \section{Related Work}
    \label{sec.related}
    \noindent
{\bf Community detection.}
Community detection \cite{fortunato2010community} partitions graphs into cohesive modules, revealing structures in social \cite{social, wang2024neural, transzero}, biological \cite{bio1, bio2}, and information systems \cite{info, enmcs}. Early studies fall into three categories. 
(1) Optimization-based methods \cite{newman2004finding, blondel2008fast, traag2019louvain} maximize metrics like modularity. 
(2) Matrix factorization methods \cite{bigclam, wang2016semantic} decompose adjacency matrices into latent representations.  
(3) Generative methods \cite{karrer2011stochastic, communitygan} assume edges are sampled from distributions. 
% Although conceptually clear, these methods rely on fixed, shallow objectives, restricting their adaptability to real-world scenarios. 
% Despite their simplicity, these methods rely on fixed objectives and shallow structural assumptions, limiting adaptability to complex real-world graphs. 
However, these methods rely on fixed objectives and often fail to adapt to real-world graphs. 
Recent learning-based methods combine representation learning with community detection.  SDCN \cite{bo2020sdcn} and DMoN \cite{tsitsulin2020dmon} jointly optimize embeddings and partitions, while CGC \cite{park2022cgc} captures patterns via structural contrasts. For targeted identification, seed-centric methods like Bespoke \cite{bakshi2018bespoke} and SEAL \cite{seal} expand from target nodes, CLARE \cite{clare} uses subgraph inference, while PROCOM \cite{procom} reduces annotation costs via pretrain-prompting. However, these methods rely on labeled data and lack transparency and interpretability.

\noindent
{\bf LM-based graph analytics.}
Existing LM-based graph analytics methods fall into four categories. 
Cascading GNN-LM models \cite{yang2021graphformers,zhao2022learning} sequentially combine GNNs and LMs to fuse structural and textual information, but incur high computational cost and lack interpretability. 
Self-supervised GNN-LM approaches \cite{chien2021node,mavromatis2023train} leverage graph-based tasks to adapt LMs to graph data, but they remain focused on node-level representation learning and provide limited interpretability of community structures. 
LLMs for graph methods \cite{guo2023gpt4graph,yuan2023evaluating,tian2024graph} transform graphs into textual descriptions for LLMs, enabling semantic reasoning but failing to explicitly model intrinsic graph structures and community boundaries. 
GraphAdapter-style methods \cite{graphadapter} attach GNN adapters to frozen PLMs and achieve strong performance on text-attributed graphs, but still rely on implicit supervision and offer limited transparency.
%%%%%% from here on talk about LUCID
%Unlike existing LLM-based graph methods that rely on embeddings or text reasoning, \ModelName uses LLMs as rule inductors to generate explicit, executable structural logic for unsupervised community detection.
% Unlike existing LLM-based graph approaches that primarily operate through implicit embeddings or text-based reasoning, \ModelName leverages LLMs as rule induction engines to generate explicit, executable structural decision logic for fully unsupervised community detection. 

\iffalse
\noindent
{\bf LLM-based graph analytics.}
Existing methods generally fall into four categories. 
Cascading GNN-LM models \cite{yang2021graphformers,zhao2022learning} sequentially fuse structural and textual information but incur high computational costs and lack interpretability. 
Self-supervised GNN-LM approaches \cite{chien2021node,mavromatis2023train} adapt LMs via graph tasks yet remain focused on node-level representations, offering limited interpretability of community structures. 
LLMs for graph methods \cite{guo2023gpt4graph,yuan2023evaluating,tian2024graph} transform graphs into text for semantic reasoning but fail to explicitly model intrinsic structures or boundaries. 
GraphAdapter-style methods \cite{graphadapter} attach GNN adapters to frozen PLMs, achieving strong performance on text-attributed graphs, but still rely on implicit supervision and offer limited transparency. 
Unlike these approaches, \ModelName employs LLMs as rule inductors to generate explicit, executable structural logic for unsupervised community detection.
\fi

    \section{Problem Definition}
    \label{sec.preliminary}
    In this section, we present the problem definition and frequently used notations throughout this paper.

\noindent
\textbf{Unsupervised global community detection.}
\textcolor{revisionpurple}{In this paper, we study unsupervised global overlapping community detection on an undirected graph $G=(V,E)$, where $V$ and $E$ denote the node set and edge set, respectively. The goal is to infer a set of potential communities $\mathcal{C}=\{C^{(i)}\}_{i=1}^{N}$ over the entire graph without queries, where each $C_i \subseteq V$, and different communities may overlap. $N$ denotes the number of communities produced by \ModelName. In contrast, $M$ denotes the number of ground-truth communities and is used only for evaluation. Notations are summarized in Table \ref{tab: notations}.}

\noindent
\textcolor{revisionpurple}{\textbf{Remark (Problem scope).}
Global community detection aims to recover all potential communities across the entire graph, whereas local community detection~\cite{bespoke,seal,clare,procom} starts from given queries and retrieves specific target communities.
Overlapping community detection~\cite{bigclam,come,communitygan,bespoke,seal,clare,procom} addresses another need by allowing each node to belong to multiple communities. 
We study the global overlapping setting because our goal is to discover all such communities without queries.
% Overlapping means that nodes can belong to multiple communities.
% Here, we study the global overlapping setting.
}

    \section{The \ModelName Framework}
    \label{sec.method}
    In this section, we present the \ModelName framework. We first provide an overview, then detail its four operational stages, and conclude with a complexity analysis.

\subsection{Framework Overview}
\label{sec: framework overview}
% We propose \ModelName, a four-stage LLM-symbolized framework for interpretable community detection, as illustrated in Figure~\ref{fig: framework}. The framework constructs communities through four stages: local-view community initialization, multi-factor community merge, multi-grain community refinement, and global-view community selection. Rather than treating LLMs as end-to-end predictors, \ModelName employs them as rule inducers in the merge and refinement stages. In the merge stage, the LLM induces a multi-factor decision tree to guide subgraph merging decisions; in the refinement stage, the LLM generates coarse-to-fine rules to remove boundary noise, converting implicit model knowledge into explicit, interpretable rules.
We propose \ModelName, a four-stage LLM-symbolized framework for interpretable community detection, as illustrated in Figure~\ref{fig: framework}. The framework constructs communities through four stages: local-view community initialization, multi-factor community merge, multi-grain community refinement, and global-view community selection. \textcolor{revisionpurple}{Rather than treating the LLM as end-to-end predictors, \ModelName employs them as rule inducers in the merge and refinement stage. Guided by our prompts, the LLM observes structural patterns from sample communities and draws on its implicit understanding of community-detection objectives to induce explicit, executable symbolic rules: a multi-factor decision tree for merging compatible subgraphs and coarse-to-fine rules for removing boundary noise.}

\begin{figure}[t!]
    \centering
    \includegraphics[width=1.0\columnwidth]{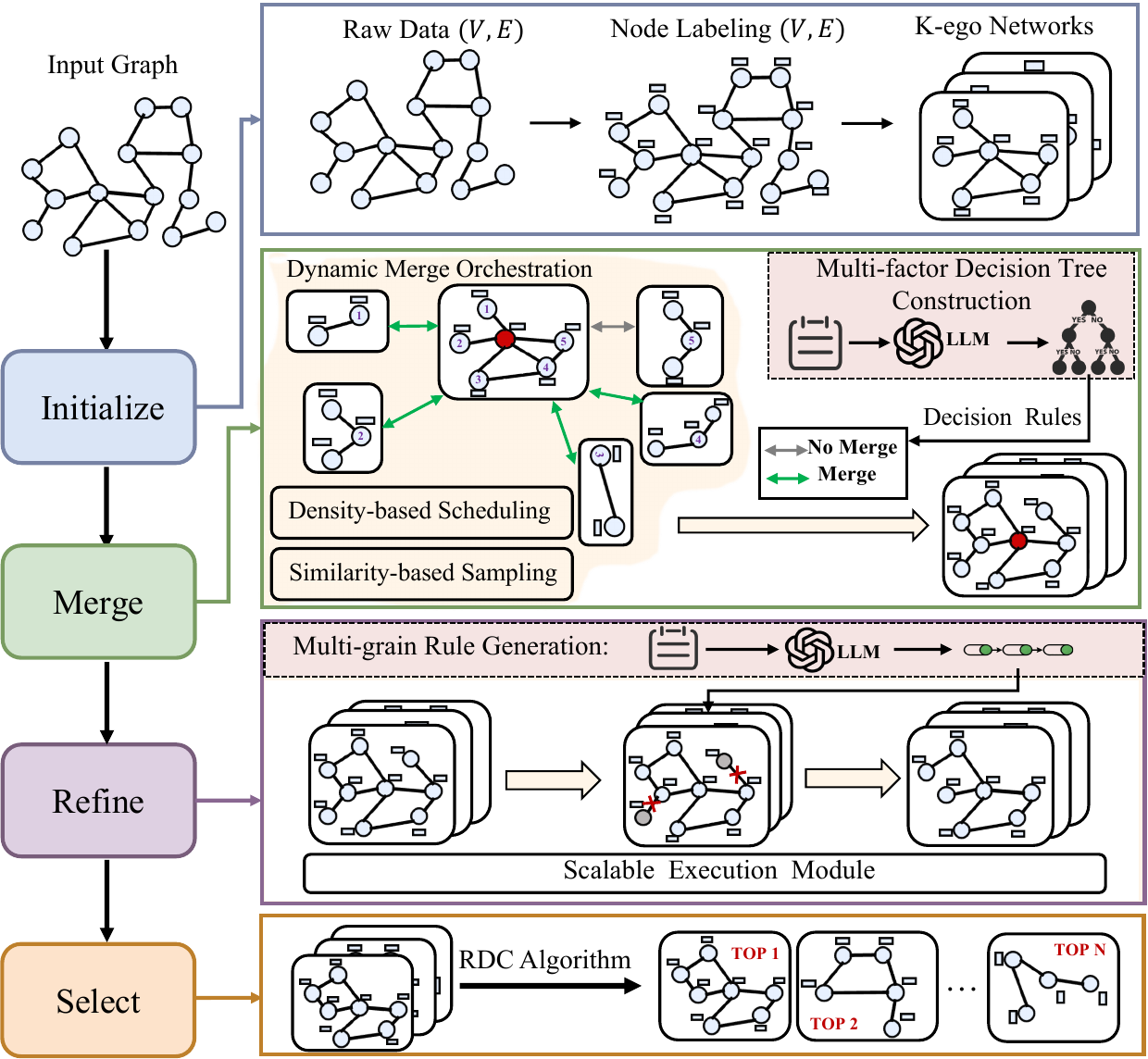}
    \caption{The framework of \ModelName: (1) The initialization stage assigns role labels to nodes and decomposes the graph into localized $k$-ego representations; (2) The merge stage leverages LLM-driven multi-factor reasoning to merge subgraphs into community candidates; (3) The refinement stage refines candidates through LLM-guided coarse-to-fine rule generation with a scalable execution module; (4) Finally, the selection stage selects final communities by balancing internal cohesion and external separation. Our LLM-symbolized framework enables interpretable community detection in a label-free manner.}
    % \caption{The framework of \ModelName. In the local-view community initialization stage, we assign unsupervised role labels to nodes and decompose the original graph into localized $k$-ego. In the multi-factor community merge stage, we leverage LLM-symbolized multi-factor reasoning to progressively merge small subgraphs into more complete community candidates. In the subsequent multi-grain community refinement stage, we further refine candidate communities through coarse-to-fine adjustments to remove boundary noise. Finally, in the global-view community selection stage, we select the final communities by balancing internal cohesion and external separation from a global structural perspective.}
    \label{fig: framework}
\end{figure}

\begin{figure*}[t!]
    \centering
    % 使用 \columnwidth 确保图片宽度不会超过单栏范围
    \includegraphics[width=1.0\textwidth]{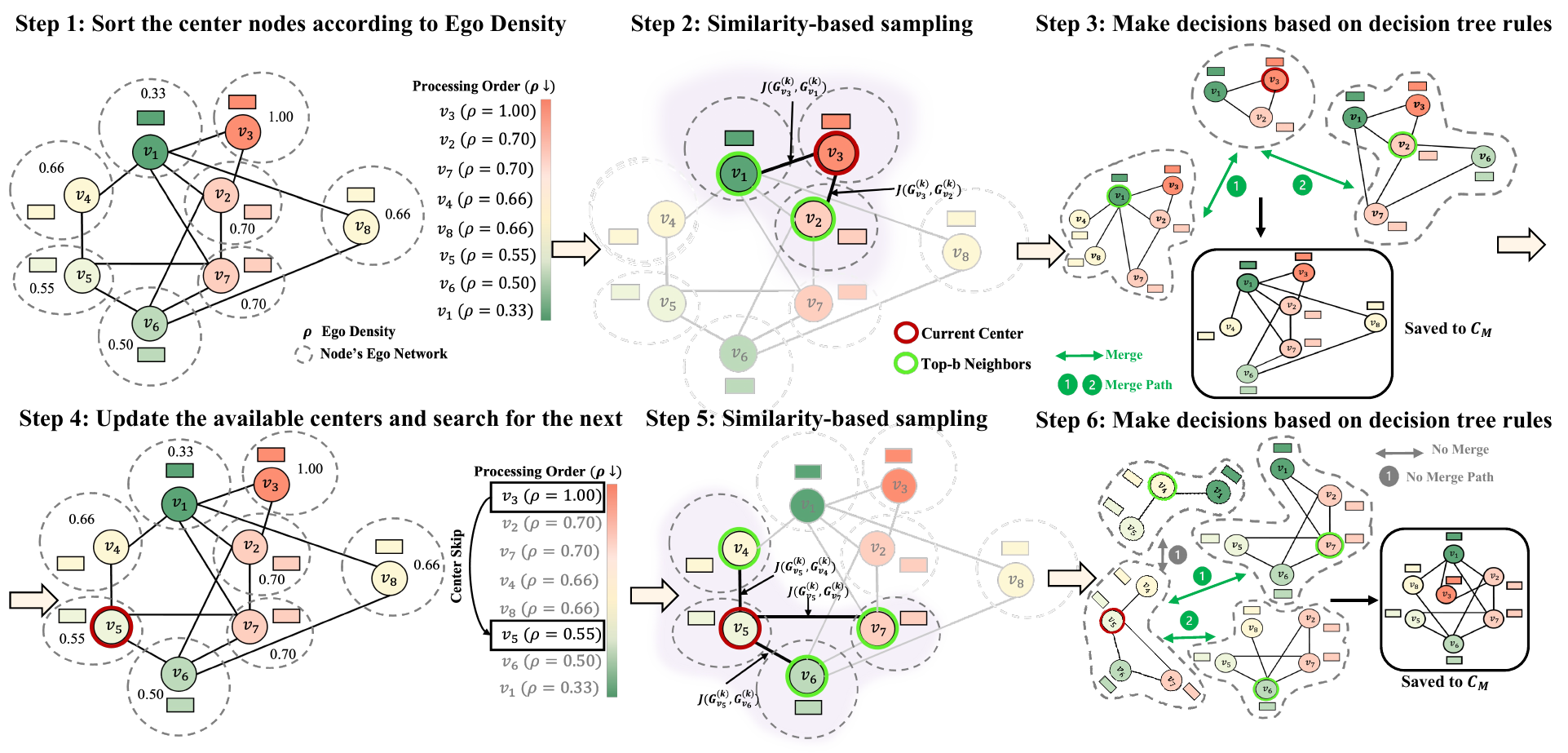}
    \caption{Illustrating the Merge Stage in \ModelName. The process iteratively merges subgraphs through six steps: (1) sorting centers by ego density, (2) sampling top-$b$ neighbors for the current center, (3) applying LLM-generated decision tree rules for merge decisions, (4) removing merged neighbors and their $k$-ego nodes from center candidates, (5) sampling neighbors for the next center, and (6) applying decision rules again. The process continues until all centers are processed, accumulating communities in $\mathcal{C}_M$. This dynamic orchestration enables effective and efficient merging at scale.}
    % \caption{Illustrating the Merge Stage in \ModelName. The process iteratively merges subgraphs through six steps: (1) sorting center nodes by ego density in descending order, (2) performing budget-guided neighbor sampling to select top-$b$ neighbors for the current center, (3) making merge decisions using LLM-generated decision tree rules, (4) updating the available center set by removing merged neighbors and all nodes within their $k$-ego networks from the center candidate set, (5) repeating budget-guided neighbor sampling for the next center, and (6) applying decision tree rules again to determine merge outcomes. The process continues until all eligible centers are processed, accumulating merged communities in $\mathcal{C}_M$. This dynamic orchestration mechanism enables effective and efficient merging at scale.}
    \label{fig:merge stage}
    \vspace{-5pt}
\end{figure*}

% \subsection{Local-view Community Initialization}
% \label{sec: preprocessing}To enable scalable processing by the LLM, we perform a $k$-ego network decomposition. This decomposition addresses two fundamental requirements. First, LLMs have limited context windows, making it infeasible to process large-scale graphs directly. By decomposing the graph into smaller, localized subgraphs, each subgraph can effectively fit within the context window of the LLM, which helps mitigate context constraints and enhances scalability. Second,
\subsection{Local-view Community Initialization}
\label{sec: preprocessing}
% {\color{red}TODO - In the Community Initialization stage, we  genereate initial communities, label nodes. For initial community generation,  we choose $k$-ego networks as the decomposition unit because they serve as natural building blocks for communities.
% Each ego network captures the local neighborhood structure centered on a node, which inherently represents a potential community seed. Each ego network $G_u^{(k)}$ is centered on node $u$ and includes all nodes within $k$ hops. This decomposition partitions the global graph into localized subgraphs ${G_u^{(k)}}_{u \in V}$, allowing LLM to process structural information in manageable units while preserving local connectivity.}
In the local-view community initialization stage, we generate initial communities and assign unsupervised role labels to nodes. 
% before revision
% To generate initial communities, we decompose the graph into $k$-ego networks. This addresses LLM context window limitations by decomposing the graph into smaller, localized subgraphs that fit within the context window. We choose $k$-ego networks as the decomposition unit because they serve as natural building blocks for communities. Each ego network $G_u^{(k)}$ is centered on node $u$ and includes all nodes within $k$ hops, partitioning the global graph into localized subgraphs ${G_u^{(k)}}_{u \in V}$ that preserve local connectivity.
\textcolor{revisionpurple}{To generate initial communities, we decompose the graph into $k$-ego networks. Every node serves as a center, yielding $|V|$ overlapping $k$-ego subgraphs $\{G_u^{(k)}\}_{u\in V}$, where each $G_u^{(k)}$ contains node $u$ and all nodes within $k$ hops. This decomposition addresses LLM context window limitations by converting the full graph into smaller structural units, while preserving local connectivity around each center. We choose $k$-ego networks as the decomposition unit because they serve as natural building blocks for communities. }

To further enhance the LLM's understanding of graph structure, we assign each node an unsupervised role label~\cite{bespoke} as an auxiliary structural signal. These labels summarize the connectivity of each node within its neighborhood, providing richer information than raw edges alone. For each node $u \in V$, we compute the Jaccard similarity between its neighborhood and those of its neighbors $v \in \mathcal{N}(u)$. The resulting set of similarity values $\mathcal{S}(u)$ is summarized using five percentiles $(Q_0, Q_{25}, Q_{50}, Q_{75}, Q_{100})$ to form a fixed-length feature vector $\mathbf{x}_u$, capturing the local connectivity structure of $u$. 
We then apply K-means clustering with $K=4$ to $\{\mathbf{x}_u\}$ to assign discrete node role labels. Each label reflects the local neighborhood structure of the corresponding node. Formally, the graph is thus represented as an attributed graph $G=(V, E)$, where each node $u \in V$ is associated with a label $l_u$ representing its identified role.
% \vspace{-15pt}
\subsection{Multi-factor Community Merge}
\label{sec: merge}
After initialization, the graph is decomposed into local subgraphs. The merge stage identifies opportunities to merge initial communities, but determining merge decisions requires considering multiple structural factors. To address this, the merge stage is organized around two key components:
(i) an LLM-guided multi-factor decision tree that determines whether candidate subgraphs should be merged, and
(ii) a dynamic merging orchestration that governs how merge decisions are executed through similarity-based sampling and density-based scheduling. The illustration of the merge stage process is shown in Figure~\ref{fig:merge stage}.
\subsubsection{Multi-factor Decision Tree Construction}
We design a structured prompt that guides the LLM to generate a multi-factor decision tree. The prompt consists of three key components: data context, decision tree definition, and output requirements. An example is given in Box \ref{box:prompt_spec}.

\noindent
{\bf Data context.}
We input a \textit{data profile} that provides dataset-level statistical metadata, together with a \textit{sample block} centered on one node, which explicitly describes its $k$-ego network and the $k$-ego networks of its direct neighbors. Formally, the input context is defined as $\mathcal{I} = \{\mathcal{I}_P, \mathcal{I}_B\}$, where $\mathcal{I}_P$ denotes the data profile, and $\mathcal{I}_B = \bigl(v_c, G_{v_c}^{(k)}, \{G_{n_j}^{(k)}\}_{j=1}^{z}\bigr)$ is a sample block. Here, $v_c$ is the center vertex, and $G_{v_c}^{(k)}$ is the $k$-ego network centered at $v_c$. $\{G_{n_j}^{(k)}\}_{j=1}^{z}$ denotes the $k$-ego networks of $z$ direct neighbors of $v_c$.

\noindent
{\bf Decision tree definition.}
We define a decision tree that hierarchically evaluates multiple structural factors (e.g., overlap, density, and labels) in a symbolic rule-based manner:
\[
\mathcal{T} = (\mathcal{U}, \mathcal{E}, r, \{\phi_n\}_{n \in \mathcal{U}_{\text{int}}}, \{\psi_\ell\}_{\ell \in \mathcal{U}_{\text{leaf}}}),
\]
% where $\mathcal{U}$ and $\mathcal{E}$ denote the sets of nodes and directed edges in the decision tree, $r$ is the root node, and $\mathcal{U}_{\text{int}}$ and $\mathcal{U}_{\text{leaf}}$ represent internal and leaf nodes, respectively. Each internal node $n \in \mathcal{U}_{\text{int}}$ is associated with a structural predicate $\phi_n: \mathcal{R}(v, u) \times \mathcal{X} \rightarrow \{0,1\}$, where $\mathcal{R}(v, u)$ denotes the structural context of a center--neighbor pair, and $\mathcal{X}$ represents the current global merge state. Each leaf node corresponds to a decision action $\psi_\ell \in \{\textsc{Merge}, \textsc{No Merge}\}$. Executing the decision tree amounts to traversing a root-to-leaf path conditioned on both local structure and global state, yielding a deterministic merge decision for each eligible center--neighbor pair.
where $\mathcal{U}$ and $\mathcal{E}$ denote the sets of nodes and directed edges in the decision tree, $r$ is the root node, and $\mathcal{U}{\text{int}}$ and $\mathcal{U}{\text{leaf}}$ represent internal and leaf nodes, respectively. Each internal node $n \in \mathcal{U}{\text{int}}$ is associated with a structural predicate $\phi_n: \mathcal{R}(v, u) \times \mathcal{X} \rightarrow {0,1}$, where $\mathcal{R}(v, u)$ denotes the structural context of a center--neighbor pair, and $\mathcal{X}$ represents the global merge state. Each leaf node corresponds to a decision action $\psi\ell \in {\textsc{Merge}, \textsc{No Merge}}$. Executing the decision tree amounts to traversing a root-to-leaf path conditioned on both local structure and global state, yielding a merge decision for each eligible center--neighbor pair.

\noindent
{\bf Output requirements.}
The LLM generates JSON-formatted decision rules which are then compiled into executable Python predicates $\phi_n$. 
% For example, for center node $v_1$ with neighbors $\{v_2, v_3\}$, the LLM generates:
% \begin{verbatim}
% {
%   "node_id": "overlap_check",
%   "code": "return jaccard(ego(v), ego(u)) >= 0.65",
%   "true_child": "density_check",
%   "false_child": "reject"
% }
% \end{verbatim}
The structured output is designed to be machine-interpretable, including an executable decision tree and corresponding decision traces for each sample. A case study of the decision is illustrated in Figure~\ref{fig:decision_tree} in Section~\ref{sec:case study}.

\subsubsection{Dynamic Merge Orchestration}
Static multi-factor decision trees are insufficient for large-scale community merging. Instead, effective merging requires dynamic orchestration to align local decisions with global consistency under sequential dependencies. 
To this end, we design two optimizations: (i) similarity-based sampling to restrict merge candidates to structurally relevant nodes under context constraints, and (ii) density-based scheduling to prioritize merges in dense regions and regulate execution order.

\noindent
{\bf Similarity-based sampling.}
We restrict merge candidates to structurally relevant neighbors using Jaccard similarity:
\[
J(v,u) = \frac{|V_{\text{ego}}(v) \cap V_{\text{ego}}(u)|}{|V_{\text{ego}}(v) \cup V_{\text{ego}}(u)|}.
\]
During decision tree construction, we select the top-$z$ neighbors to form the sample block $\mathcal{I}_B$; during merging, for each center node $v$, we select the top-$b$ neighbors $\mathcal{N}_v^{(b)}$. This restricts merge decisions to relevant neighbors and maintains prompt length within the LLM's context window.
% \noindent
% {\bf Similarity-based sampling.}
% For a center node $v$, we only consider neighbors $u$ whose $k$-ego graphs are similar enough, measured by the Jaccard similarity:
% \[
% J(v,u) = \frac{|V_{\text{ego}}(v) \cap V_{\text{ego}}(u)|}{|V_{\text{ego}}(v) \cup V_{\text{ego}}(u)|}.
% \]
% similarity-based sampling operates in two stages: (i) \textit{Decision tree construction}: when constructing the prompt for LLM decision tree generation, we select the top-$z$ neighbors with the highest Jaccard similarity to form the sample block $\mathcal{I}_B$. (ii) \textit{Dynamic merge orchestration}: during the actual merging process, for each center node $v$, we select the top-$b$ neighbors with the highest similarity, denoted as $\mathcal{N}_v^{(b)}$. This two-stage approach restricts merge decisions to structurally relevant neighbors, reduces noise, and helps maintain the prompt length within the LLM's context window.

\noindent
{\bf Density-based scheduling.}
Let $\mathcal{A}$ be the set of available center nodes. Each node $v \in \mathcal{A}$ has an ego graph density defined as
\[
\text{EgoDensity}(v) = \frac{2|E_{\text{ego}}(v)|}{|V_{\text{ego}}(v)|(|V_{\text{ego}}(v)|-1)}.
\]
Nodes are sorted by density in descending order and processed sequentially. This density-based ordering prioritizes nodes in structurally rich regions first, ensuring high-coverage merges of core communities while deferring sparse or peripheral nodes to reduce unnecessary computations. For each neighbor $u \in \mathcal{N}_b(v)$, the merge decision $M(v,u) \in \{0,1\}$ indicates acceptance ($1$) or rejection ($0$). If $M(v,u)=1$, the available set is updated as
\[
\mathcal{A} \leftarrow \mathcal{A} \setminus V_{\text{ego}}(u),
\]
so that $u$ and its ego nodes cannot act as centers again but can still be merged into other communities. The overall process can be expressed as
\[
\text{for } v \in \text{sort}(\mathcal{A}, \text{EgoDensity}), \quad 
\mathcal{C}_v \gets \mathcal{C}_v \cup \{ u \in \mathcal{N}_b(v) : M(v,u)=1 \}.
\]
Together, these two mechanisms control which merges happen and the order of execution, ensuring stable local merges while preserving global community consistency in large-scale graphs.

\subsection{Multi-grain Community Refinement}
\label{sec: refine}
After merging, community backbones are recovered through subgraph-level aggregation. The refinement stage corrects ambiguous boundary vertices by refining individual node memberships using a coarse-to-fine rule set, organized into two phases: (i) multi-grain rule generation, establishing a multi-layered evaluation standard from global structural consistency to local connection strength; (ii) parallel rule execution, leveraging the independence of node-level decisions to enable efficient batch processing.

% \subsubsection{Multi-grain Rule Generation}
% We craft a tailored prompt to instruct the LLM in generating a hierarchy of coarse-to-fine node refinement rules. The prompt consists of three key components: data context, coarse-to-fine rule set specification, and output requirements. An example is given in Box \ref{box:prompt_spec2}.

% \noindent
% {\bf Data context.}
% We input a \textit{data profile} that provides dataset-level statistical metadata, together with a \textit{sample block}, which refers to example candidate communities formed during the merging phase.

% \noindent
% {\bf Multi-grain rule design.}
% The LLM generates refinement rules organized in a coarse-to-fine hierarchy: (1) \textit{Coarse rules} filter clearly inconsistent nodes using broad criteria: for example, "Remove nodes with $<20\%$ internal edges relative to their degree." ; (2) \textit{Medium-grained rules} examine local neighborhood structure: "Remove nodes whose role label differs from $80\%+$ of their neighbors.";
% (3) \textit{Fine-grained rules} apply precise structural metrics: "Remove nodes with local clustering coefficient $<0.5$ AND betweenness centrality $>$ threshold." This progressive filtering reduces refinement cost by approximately $60-70\%$ compared to applying fine-grained rules to all nodes, as most outliers are eliminated by coarse checks before expensive fine-grained evaluation.

% \noindent
% {\bf Output requirements.}
% The refinement procedure outputs, for each community, a list of nodes to drop, limited to outliers or structural inconsistencies.
\subsubsection{Multi-grain Rule Generation}
We prompt the LLM to generate a hierarchy of coarse-to-fine node refinement rules using three components: data context, rule-set specification, and output requirements. An example is given in Box \ref{box:prompt_spec2}.

\noindent
{\bf Data context.}
The input comprises a \textit{data profile} with dataset-level statistical metadata and a \textit{sample block} of example candidate communities formed during merging.

\noindent
{\bf Multi-grain rule design.}
The LLM generates a coarse-to-fine rule hierarchy: (1) \textit{Coarse rules} filter clear inconsistencies using broad criteria (e.g., nodes with $<20\%$ internal edges relative to their degree); (2) \textit{Medium-grained rules} examine local neighborhoods (e.g., nodes whose role label differs from $80\%+$ of their neighbors); and (3) \textit{Fine-grained rules} apply precise structural metrics (e.g., local clustering coefficient $<0.5$ and betweenness centrality above a threshold). This progressive filtering reduces refinement cost by approximately $60-70\%$ compared with applying fine-grained rules to all nodes, as coarse checks eliminate most outliers before expensive fine-grained evaluation.

\noindent
{\bf Output requirements.}
For each community, refinement outputs a list of nodes to drop, restricted to outliers or structural inconsistencies.

% The refinement procedure outputs, for each candidate community, a list of nodes to drop, limited to those confidently identified as outliers or structural inconsistencies. This results in a precise and conservative representation suitable for downstream tasks.

\subsubsection{Parallel Rule Execution}
% Our design decouples node-level refinement by ensuring they are conditionally independent given a community. This architectural choice unlocks dual-level parallelism: multiple communities can be processed concurrently, while nodes within each community are refined in batch using a shared, LLM-generated rule set. By applying these multi-granularity rules progressively from coarse to fine, our method efficiently prunes nodes that deviate from the community's structural and relational semantics. Consequently, the refinement stage introduces minimal computational overhead, scaling linearly with the node count and operating orders of magnitude faster than the preceding subgraph merging phase.
Our design makes node-level refinement decisions conditionally independent within each community, enabling two levels of parallelism: multiple communities are processed simultaneously, while nodes within each community are refined in batches using a shared LLM-generated rule set. 
Applied progressively from coarse to fine, these rules efficiently remove nodes that deviate from a community’s structural and relational patterns. 
Consequently, refinement incurs minimal computational overhead, scales linearly with the number of nodes, and runs orders of magnitude faster than the preceding subgraph-merging stage.
% Our design decouples node-level refinement by making decisions conditionally independent within each community. 
% This enables two levels of parallelism: multiple communities can be processed simultaneously, and nodes within each community are refined in batches using a shared LLM-generated rule set. 
% These rules are applied progressively from coarse to fine. This allows the method to efficiently remove nodes that deviate from a community’s structural and relational patterns. 
% Consequently, refinement incurs minimal computational overhead. The runtime increases linearly with the number of nodes and is orders of magnitude faster than the earlier subgraph-merging stage.

\begin{table*}[htbp!]
    \centering
    \caption{Overall performance comparison against unsupervised and semi-supervised baselines (results in percent $\pm$ standard deviation). ``N/A'' denotes the algorithm failing to converge within 4 days. The best results are highlighted in bold, and the strongest baseline results are underlined. With a fixed prompt, \ModelName reports mean$\pm$std over three independent LLM generations. Our method is purely unsupervised, yet it outperforms the state-of-the-art semi-supervised baseline.}
    \label{tab:full_comparison}
    \resizebox{\textwidth}{!}{
    \begin{tabular}{c|c|ccc|cccc|cc}
    \toprule
    \multirow{2}{*}{\textbf{Metric}} & \multirow{2}{*}{\textbf{Dataset}} & \multicolumn{3}{c|}{\textbf{Unsupervised}} & \multicolumn{4}{c|}{\textbf{Semi-supervised}} & \textbf{Unsupervised} & \multirow{2}{*}{\textbf{Improv.}} \\
    \cmidrule(lr){3-5} \cmidrule(lr){6-9} \cmidrule(lr){10-10}
    & & BigClam & ComE & CommGAN & Bespoke & SEAL & CLARE & PROCOM & \textbf{\ModelName} & \\
    \midrule
    \multirow{5}{*}{F1 Score} %
     & Facebook    & \textcolor{revisionpurple}{$36.59_{\pm 0.23}$} & \textcolor{revisionpurple}{\underline{$38.18_{\pm 0.64}$}} & \textcolor{revisionpurple}{$25.84_{\pm 0.99}$} & \textcolor{revisionpurple}{$29.29_{\pm 3.52}$} & \textcolor{revisionpurple}{$12.15_{\pm 6.31}$} & \textcolor{revisionpurple}{$23.45_{\pm 0.95}$} & \textcolor{revisionpurple}{$34.97_{\pm 1.19}$} & \textcolor{revisionpurple}{\textbf{40.82}$_{\pm 0.54}$} & \textcolor{revisionpurple}{+6.9\%} \\
     & Amazon      & \textcolor{revisionpurple}{$71.43_{\pm 0.05}$} & \textcolor{revisionpurple}{$78.32_{\pm 0.20}$} & \textcolor{revisionpurple}{$70.59_{\pm 0.76}$} & \textcolor{revisionpurple}{$65.65_{\pm 2.21}$} & \textcolor{revisionpurple}{$79.35_{\pm 0.45}$} & \textcolor{revisionpurple}{$78.38_{\pm 2.48}$} & \textcolor{revisionpurple}{\underline{$83.42_{\pm 0.27}$}} & \textcolor{revisionpurple}{\textbf{92.11}$_{\pm 0.78}$} & \textcolor{revisionpurple}{+10.4\%} \\
     & Livejournal & \textcolor{revisionpurple}{$40.75_{\pm 0.31}$} & N/A   & \textcolor{revisionpurple}{$19.68_{\pm 15.44}$} & \textcolor{revisionpurple}{$36.82_{\pm 1.53}$} & \textcolor{revisionpurple}{\underline{$52.95_{\pm 7.12}$}} & \textcolor{revisionpurple}{$41.50_{\pm 1.97}$} & \textcolor{revisionpurple}{$51.06_{\pm 1.55}$} & \textcolor{revisionpurple}{\textbf{54.25}$_{\pm 1.23}$} & \textcolor{revisionpurple}{+2.5\%} \\
     & DBLP        & \textcolor{revisionpurple}{$40.29_{\pm 0.08}$} & N/A & N/A   & \textcolor{revisionpurple}{$44.50_{\pm 2.06}$} & \textcolor{revisionpurple}{$17.98_{\pm 1.64}$} & \textcolor{revisionpurple}{\underline{$50.45_{\pm 0.53}$}} & \textcolor{revisionpurple}{$47.74_{\pm 3.94}$} & \textcolor{revisionpurple}{\textbf{54.73}$_{\pm 0.50}$} & \textcolor{revisionpurple}{+8.5\%} \\
     & Twitter     & \textcolor{revisionpurple}{\underline{$29.55_{\pm 0.05}$}} & \textcolor{revisionpurple}{$25.92_{\pm 0.14}$} & N/A   & \textcolor{revisionpurple}{$28.96_{\pm 0.88}$} & \textcolor{revisionpurple}{$12.83_{\pm 0.39}$} & \textcolor{revisionpurple}{$17.26_{\pm 0.90}$} & \textcolor{revisionpurple}{$24.69_{\pm 1.12}$} & \textcolor{revisionpurple}{\textbf{32.51}$_{\pm 0.31}$} & \textcolor{revisionpurple}{+10.0\%} \\
    \midrule
    \multirow{5}{*}{Jaccard} 
     & Facebook    & \textcolor{revisionpurple}{$26.11_{\pm 0.22}$} & \textcolor{revisionpurple}{\underline{$28.36_{\pm 0.44}$}} & \textcolor{revisionpurple}{$17.68_{\pm 1.03}$} & \textcolor{revisionpurple}{$20.25_{\pm 3.37}$} & \textcolor{revisionpurple}{$7.34_{\pm 4.69}$} & \textcolor{revisionpurple}{$15.91_{\pm 0.60}$} & \textcolor{revisionpurple}{$25.09_{\pm 0.93}$} & \textcolor{revisionpurple}{\textbf{31.77}$_{\pm 0.61}$} & \textcolor{revisionpurple}{+12.0\%} \\
     & Amazon      & \textcolor{revisionpurple}{$61.44_{\pm 0.07}$} & \textcolor{revisionpurple}{$69.93_{\pm 0.24}$} & \textcolor{revisionpurple}{$62.47_{\pm 0.71}$} & \textcolor{revisionpurple}{$59.54_{\pm 2.75}$} & \textcolor{revisionpurple}{$69.34_{\pm 0.93}$} & \textcolor{revisionpurple}{$69.51_{\pm 3.04}$} & \textcolor{revisionpurple}{\underline{$74.89_{\pm 0.23}$}} & \textcolor{revisionpurple}{\textbf{87.41}$_{\pm 1.04}$} & \textcolor{revisionpurple}{+16.7\%} \\
     & Livejournal & \textcolor{revisionpurple}{$31.77_{\pm 0.30}$} & N/A   & \textcolor{revisionpurple}{$16.88_{\pm 13.18}$} & \textcolor{revisionpurple}{$30.67_{\pm 1.63}$} & \textcolor{revisionpurple}{$41.60_{\pm 7.19}$} & \textcolor{revisionpurple}{$33.09_{\pm 2.16}$} & \textcolor{revisionpurple}{\underline{$41.71_{\pm 1.43}$}} & \textcolor{revisionpurple}{\textbf{46.97}$_{\pm 1.15}$} & \textcolor{revisionpurple}{+12.6\%} \\
     & DBLP        & \textcolor{revisionpurple}{$28.61_{\pm 0.09}$} & N/A & N/A   & \textcolor{revisionpurple}{\underline{$39.24_{\pm 2.28}$}} & \textcolor{revisionpurple}{$11.64_{\pm 1.04}$} & \textcolor{revisionpurple}{$39.03_{\pm 0.49}$} & \textcolor{revisionpurple}{$36.42_{\pm 4.00}$} & \textcolor{revisionpurple}{\textbf{45.20}$_{\pm 0.60}$} & \textcolor{revisionpurple}{+15.2\%} \\
     & Twitter     & \textcolor{revisionpurple}{$18.72_{\pm 0.04}$} & \textcolor{revisionpurple}{$16.43_{\pm 0.12}$} & N/A   & \textcolor{revisionpurple}{\underline{$19.51_{\pm 0.62}$}} & \textcolor{revisionpurple}{$7.25_{\pm 0.21}$} & \textcolor{revisionpurple}{$10.55_{\pm 0.56}$} & \textcolor{revisionpurple}{$16.30_{\pm 0.73}$} & \textcolor{revisionpurple}{\textbf{21.83}$_{\pm 0.25}$} & \textcolor{revisionpurple}{+11.9\%} \\
    \bottomrule
    \end{tabular}
    }
\end{table*}

\subsection{Global-view Community Selection}
\label{sec: rank}
Since the refinement stage produces a large pool of candidate communities, we then require a structural quality metric to identify high-quality communities. Existing ranking metrics such as modularity and density have notable limitations: modularity suffers from resolution limits~\cite{wang2024neural}, while density-based metrics overlook community boundaries~\cite{flaw_density}. To address this, we utilize local conductance to balance internal cohesion with external separation. \textcolor{revisionpurple}{However, standard conductance remains bounded for communities with the minimum number of internal edges required for connectivity, as formalized below.}
\textcolor{revisionpurple}{\begin{lemma}[Limitation of Standard Conductance]
\label{lem:conductance_limit}
Let $G=(V,E)$ be a graph and let $\mathcal{C}\subseteq V$ be a connected community with $n=|\mathcal{C}|\geq 2$. Denote its complement by $\bar{\mathcal{C}}=V\setminus\mathcal{C}$, its set of internal edges by $E(\mathcal{C})$, and the number of edges crossing from $\mathcal{C}$ to $\bar{\mathcal{C}}$ by $Cut(\mathcal{C},\bar{\mathcal{C}})$. Let $e=|E(\mathcal{C})|$ and $c=Cut(\mathcal{C},\bar{\mathcal{C}})$, and define the standard conductance score as $S_{base}(\mathcal{C})=c/e$. If the subgraph induced by $\mathcal{C}$ is a tree, then $e=n-1$ and $S_{base}(\mathcal{C})=c/e=c/(n-1)<\infty$.
\end{lemma}
}
% \textcolor{revisionpurple}{\begin{lemma}[Limitation of Standard Conductance]
% \label{lem:conductance_limit}
% Let $G=(V,E)$ be a graph and let $\mathcal{C}\subseteq V$ be a connected community with $n=|\mathcal{C}|\geq 2$. Denote its complement by $\bar{\mathcal{C}}=V\setminus\mathcal{C}$, its set of internal edges by $E(\mathcal{C})$, and the number of edges crossing from $\mathcal{C}$ to $\bar{\mathcal{C}}$ by $Cut(\mathcal{C},\bar{\mathcal{C}})$. Let $e=|E(\mathcal{C})|$ and $c=Cut(\mathcal{C},\bar{\mathcal{C}})$, and define the standard conductance score as $S_{base}(\mathcal{C})=c/e$. If the subgraph induced by $\mathcal{C}$ is a tree, then $e=n-1$ and
% \begin{equation}
%     S_{base}(\mathcal{C})=\frac{c}{e}=\frac{c}{n-1}<\infty.
% \end{equation}
% \end{lemma}
% }
To address this limitation, we propose RDC (Ranking with Density and Conductance), a density-regularized conductance scoring function that incorporates a penalty derived from the minimum spanning tree constraint:
\begin{equation}
\label{eq:penalized_score}
RDC(\mathcal{C_R}) = \frac{Cut(\mathcal{C_R}, \bar{\mathcal{C_R}})}{|E(\mathcal{C_R})| - \epsilon(|\mathcal{C_R}| - 1)} .
\end{equation}
Here, $Cut(\mathcal{C_R}, \bar{\mathcal{C_R}})$ measures external separation, $|E(\mathcal{C_R})|$ is the number of internal edges, and $\epsilon(|\mathcal{C_R}| - 1)$ introduces a structural penalty against chain-like communities. We compute $RDC(\mathcal{C})$ for each candidate community, rank them in ascending order, and select the top-$N$ communities with the lowest scores.

The asymptotic behavior and edge-count sensitivity of RDC relative to standard conductance are given by the following theorem.

\begin{theorem}[Properties of RDC]
\label{thm:score_advantage}
Let $S_{new}(e)=c/(e-\eta)$ and $S_{base}(e)=c/e$, where $c>0$, $\eta=\epsilon(|\mathcal{C}|-1)>0$, and $e>\eta$. Under the continuous relaxation of $e$, the following properties hold:

\begin{enumerate}[leftmargin=*]
    \item $\displaystyle \lim_{e\to\eta^+}S_{base}(e)=\frac{c}{\eta}$, whereas $\displaystyle \lim_{e\to\eta^+}S_{new}(e)=+\infty$.
    
    \item The absolute derivative of $S_{new}$ with respect to $e$ is strictly larger than that of $S_{base}$:
    \begin{equation}
        \left| \frac{d S_{new}}{d e} \right| = \frac{c}{(e-\eta)^2} > \frac{c}{e^2} = \left| \frac{d S_{base}}{d e} \right|
    \end{equation}
    for all $e>\eta$.
\end{enumerate}
\end{theorem}

\begin{proof}
For the first property, $e\to\eta^+$ gives $c/e\to c/\eta$, while $e-\eta\to0^+$ gives $c/(e-\eta)\to+\infty$. For the second property, differentiation yields $dS_{new}/de=-c/(e-\eta)^2$ and $dS_{base}/de=-c/e^2$. Since $e>\eta>0$, we have $0<e-\eta<e$ and hence $(e-\eta)^2<e^2$. Together with $c>0$, this gives $c/(e-\eta)^2>c/e^2$.
\end{proof}

\subsection{Complexity Analysis}
\label{sec: complexity analysis}
\textcolor{revisionpurple}{The overall time complexity of \ModelName is
$O(n\bar{d}^{k} + nbD + |\mathcal{C}_M|\bar{s}R
+ |\mathcal{C}_R|\bar{s}\bar{d})$.
Detailed stage-wise derivations and the space complexity analysis are provided in Appendix~\ref{sec:complexity_appendix}.}

    \section{Experiments}
    \label{sec.experiment}
    \subsection{Experimental setup}
% In this section, we present the experimental setup and report comprehensive empirical results. The experiments are designed to evaluate \ModelName systematically and to address the following research questions (RQs):
In this section, we describe the experimental setup and present comprehensive results to systematically evaluate \ModelName and address the following research questions (RQs):

\begin{itemize}
    \item \textbf{RQ1 (Overall performance)} As a fully unsupervised method, how does \ModelName perform compared with other unsupervised methods and representative semi-supervised baselines?

    \item \textbf{RQ2 (Ablation study)} How do the each key component of \ModelName contributes to the overall performance?

    \item \textbf{RQ3 (Efficiency study)} How efficient is \ModelName compared to other methods?
    
    \item \textbf{RQ4 (Hyperparameter sensitivity)} How sensitive is \ModelName to critical hyperparameters? 
    % Does the performance remain stable within reasonable parameter ranges?
  
    \item \textbf{RQ5 (Case study)} How does \ModelName accurately identify communities while providing interpretability? 
    % Specifically, what is the structure of the decision trees generated during the merge stage?
\end{itemize}
\subsubsection{Datasets.}
% Following previous works~\cite{procom,seal,clare}, we evaluate \ModelName on five real-world networks from SNAP: Facebook, Amazon, Livejournal, DBLP, and Twitter. These datasets span social, e-commerce, and academic networks and contain overlapping communities with partial labels. 
% Detailed dataset statistics are provided in Table ~\ref{tab: dataset} in Appendix~\ref{sec:datasets}. We follow the preprocessing protocols used by ProCom, CLARE, and SEAL to remove outliers and ensure a fair evaluation.
Following previous works \cite{seal,clare,procom}, we evaluate our method on five real-world datasets from SNAP\footnote{\url{http://snap.stanford.edu/data/}} (see Table \ref{tab:dataset}), covering social networks (Facebook, Twitter, Livejournal), e-commerce (Amazon), and academic collaboration (DBLP). 
\textcolor{revisionpurple}{We follow the dataset preprocessing protocols used by ProCom, CLARE, and SEAL to remove outliers and ensure a fair evaluation.}
All datasets contain overlapping communities with partial labels. 
% Following previous works \cite{seal,clare,procom}, we evaluate our method on five widely-used real-world datasets from the SNAP\footnote{\url{http://snap.stanford.edu/data/}}(see Table \ref{tab:dataset} for detailed statistics), covering diverse domains such as social networks (Facebook, Twitter, and Livejournal), e-commerce (Amazon), and academic collaboration (DBLP). Specifically, Facebook and Twitter consist of ego-networks with user-defined circles; Amazon is a product co-purchasing network categorized by product types; DBLP is a co-authorship network grouped by research fields; and Livejournal is a blogging community based on user-joined groups. All datasets contain overlapping communities and are partially labeled, meaning only a subset of nodes is assigned to ground-truth communities.
\begin{table}[t!]
    \centering
    \caption{Summary of dataset characteristics. The symbols $|V|$, $|E|$, and $|C|$ represent the total number of nodes, edges, and communities, while $|\overline{C}|$ indicates the mean community size, and $\bar{d}$ denotes the average degree.
    %\textcolor{red}{usually we use |V| or |E| or m/n to represent number of nodes and edges}
    }
    \label{tab:dataset}
    \vspace{1 em}
    \begin{tabular}{l | ccccc}
        \toprule
        \rowcolor[gray]{0.92} \textbf{Dataset} & \textbf{$|V|$} & \textbf{$|E|$} & \textbf{$|C|$} & $|\overline{C}|$ & \textbf{$\bar{d}$} \\
        \midrule
        Facebook    & 3,622   & 72,964    & 130   & 15.6 & 40.3 \\
        Amazon      & 13,178  & 33,767    & 4,517 & 9.3  & 5.1  \\
        Livejournal & 69,860  & 911,179   & 1,000 & 13.0 & 26.1 \\
        DBLP        & 114,095 & 466,761   & 4,559 & 8.4  & 8.2  \\
        Twitter     & 87,760  & 1,293,985 & 2,838 & 10.9 & 29.5 \\
        \bottomrule
    \end{tabular}
    \vspace{5pt}
\end{table}

\subsubsection{Baselines.}
To show the effectiveness of \ModelName, we compare it with both representative unsupervised and semi-supervised community detection methods. \textcolor{revisionpurple}{Following ProCom, we evaluate global and local methods under consistent settings, including both unsupervised and semi-supervised cases.} Unsupervised baselines include BigClam \cite{bigclam}, which uses matrix factorization to detect overlapping communities; ComE \cite{come}, which jointly learns node embeddings and community structures; and CommunityGAN \cite{communitygan}, which extends BigClam by using motifs instead of simple edges. Semi-supervised baselines include Bespoke \cite{bespoke}, which uses community size and structural information; SEAL \cite{seal}, which employs GANs to learn from labels; CLARE \cite{clare}, which uses a locator and a rewriter; and PROCOM \cite{procom}, which adopts a ``pretrain, prompt'' paradigm with dual-level pre-training.

\subsubsection{Evaluation metrics.}
% The performance is evaluated using the average similarity score based on bi-directional matching. Let $\dot{\mathcal{C}}$ and $\hat{\mathcal{C}}$ be the sets of ground-truth and predicted communities, respectively. The final score is defined as:
% \begin{equation}
%     \text{Score} = \frac{1}{2} \left( \frac{1}{|\hat{\mathcal{C}}|} \sum_{\hat{C} \in \hat{\mathcal{C}}} \max_{\dot{C} \in \dot{\mathcal{C}}} \text{Sim}(\hat{C}, \dot{C}) + \frac{1}{|\dot{\mathcal{C}}|} \sum_{\dot{C} \in \dot{\mathcal{C}}} \max_{\hat{C} \in \hat{\mathcal{C}}} \text{Sim}(\dot{C}, \hat{C}) \right)
% \end{equation}
% where $\text{Sim}(\cdot, \cdot)$ can be either the F1-score or Jaccard index.
For graph networks with available ground-truth communities, performance is typically evaluated using F1 and Jaccard scores \cite{bigclam,procom,seal,clare}based on bi-directional matching. Suppose there are $M$ ground-truth communities ${\dot{C}^{(i)}}$ and $N$ communities predicted by the model ${\hat{C}^{(j)}}$. For each predicted community, we compute its similarity to the most similar ground-truth community, and vice versa. The final score is obtained by averaging the two directions:
\begin{equation}
\frac{1}{2}\left(
\frac{1}{N}\sum_{j}\max_{i}\delta(\hat{C}^{(j)}, \dot{C}^{(i)})
+
\frac{1}{M}\sum_{i}\max_{j}\delta(\hat{C}^{(j)}, \dot{C}^{(i)})
\right),
\end{equation}
\vspace{-2pt}
where $\delta$ represents either the F1 score or the Jaccard score.
\subsubsection{Implementation Details.}
All experiments are conducted on an NVIDIA A100. We use GPT-5.2 as the LLM backbone for our main experiments. The number of predicted communities $N$ is set to 200 for Facebook, 1000 for Livejournal, 1500 for Amazon, and 5000 for DBLP and Twitter.
\textcolor{revisionpurple}{We use the official APIs with their default temperatures: 1.0 for GPT-5.2 and DeepSeek-V4-Pro, and 0.8 for Qwen3-Max.}

% Although we access DeepSeek-V4-Pro through its official API, its model weights are publicly available under the MIT License and support self-hosted inference~\cite{deepseekv4}, reducing dependence on continued API availability.
% \textcolor{revisionpurple}{We use the default settings of official LLM APIs, with temperature 1.0 for GPT-5.2 and DeepSeek-V4-Pro and 0.8 for Qwen3-Max.
% For each main experiment, we fix the prompt and report mean$\pm$std over three independent LLM generations.}
% The source code will be released at~\url{https://anonymous.4open.science/r/KDD2026LUCID-A702}.

\subsection{Overall performance (RQ1)}
\label{sec:overall pfms}
\textcolor{revisionpurple}{We provide the overall performance comparison in Table~\ref{tab:full_comparison}, where the reported \ModelName results use GPT-5.2 as the LLM backbone. Results with the open-weight DeepSeek-V4-Pro~\cite{deepseekv4} are reported in Appendix~\ref{sec:deepseek_results}}.
Following previous works ~\cite{procom,seal,clare}, we report the performance of unsupervised community detection algorithms without variation. 
Similarly, for semi-supervised algorithms, we randomly select 10 communities for training (prompts for PROCOM) and use the remaining communities for testing. 
% This ensures a fair comparison and maintains consistency with the established evaluation protocol. 
% Based on the results in Table~\ref{tab:full_comparison}, we note the following key observations:
We observe the following patterns: (1) \ModelName significantly outperforms existing unsupervised methods across all datasets.
Compared with the best unsupervised baseline on each dataset, our method achieves average improvements of \textcolor{revisionpurple}{20.7\%} in F1 score and \textcolor{revisionpurple}{31.9\%}in Jaccard score.
(2) As an unsupervised method, \ModelName consistently outperforms state-of-the-art semi-supervised baselines. 
Specifically, \ModelName achieves the best performance across all five datasets, with average relative gains of \textcolor{revisionpurple}{10.1\%} in F1 and \textcolor{revisionpurple}{16.6\%} in Jaccard over the strongest semi-supervised baseline on each dataset. This shows that \ModelName captures community structures from network topology without any labels.
(3) While semi-supervised models like PROCOM occasionally achieve second-best results, they rely on labelled data and training. 
When compared with the strongest baseline in each row, including both unsupervised and semi-supervised methods, \ModelName achieves average relative gains of 7.7\% in F1 and 13.7\% in Jaccard.
% We observe the following patterns:(1) \ModelName significantly outperforms existing unsupervised methods across all datasets. 
% Compared to the best unsupervised baseline, our method achieves an average improvement of 22.2\% in F1 score and 19.2\% in Jaccard score. 
% (2) As an unsupervised method, \ModelName consistently outperforms state-of-the-art semi-supervised baselines. 
% Specifically, \ModelName achieves the best performance across all five datasets, with average gains of 8.3\% F1 and 12.0\% Jaccard. This is shows that \ModelName captures community structures from network topology without any labels.
% (3) While semi-supervised models like PROCOM occasionally achieve second-best results, they rely on labelled data and training. 

\begin{figure}[t!]
    \centering
    \includegraphics[width=0.45\textwidth]{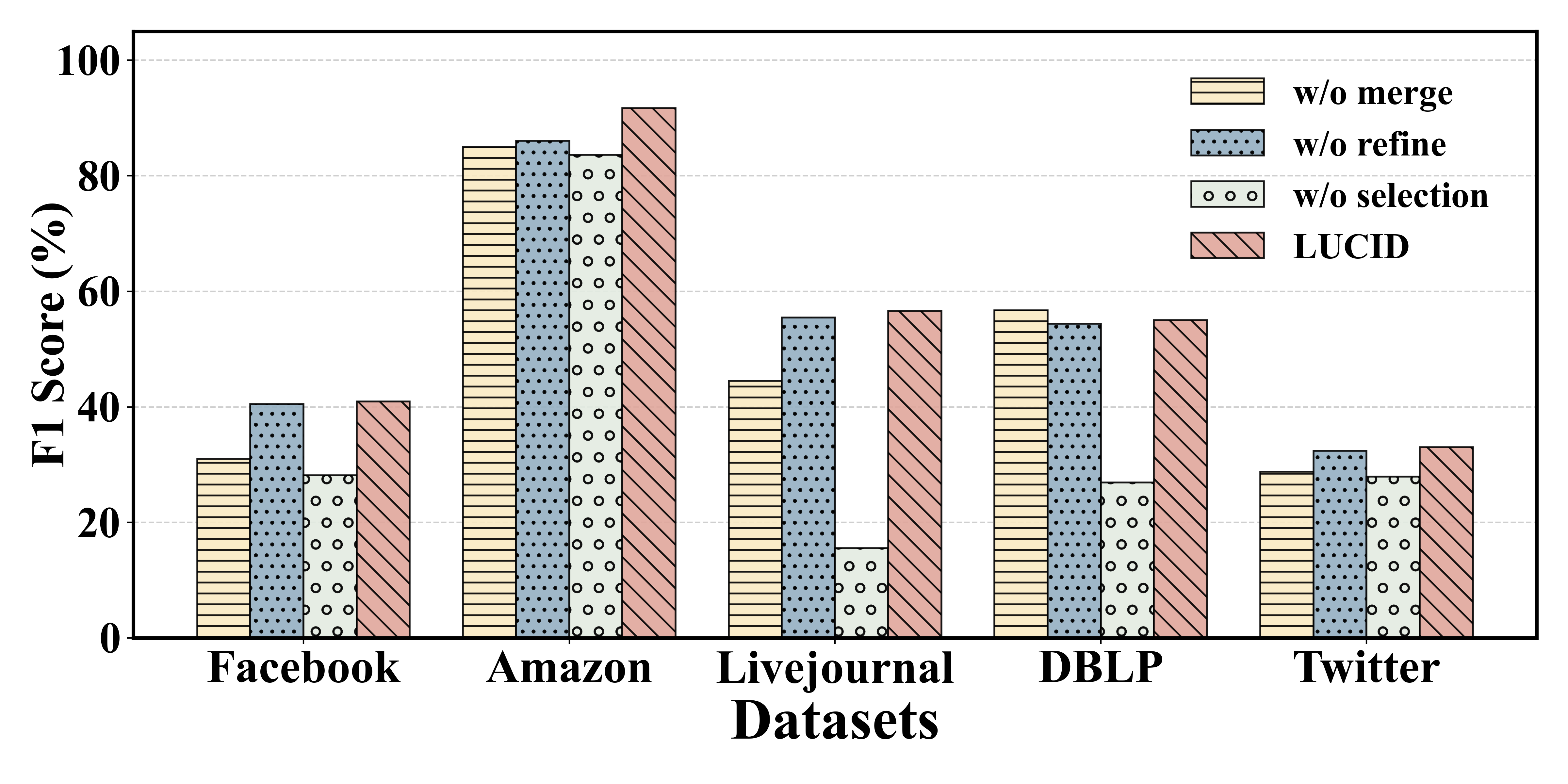} 
    \caption{Ablation study on the effectiveness of the four stages in \ModelName.}
    \label{fig:ablation}
\end{figure}

\begin{table}[t!]
\centering
\small
\setlength{\tabcolsep}{6pt}
\caption{Ablation study with different rank metrics on F1 score (\%) across five datasets.}
\label{tab:ablation_rank}
\begin{tabular}{lcccc}
\toprule
\textbf{Dataset} & \textbf{Density} & \textbf{Modularity} & \textbf{LC} & \textbf{RDC} \\
\midrule
Facebook     & 30.27 & 24.54 & 41.00 & 40.95 \\
Amazon  & 44.03 & 60.41 & 86.54 & 91.71 \\
Livejournal & 9.35 & 17.70 & 52.64 & 56.57 \\
DBLP   & 22.88 & 35.76 & 54.28 & 55.03 \\
Twitter & 27.98 & 26.46 & 32.71 & 33.03 \\
\midrule
\textbf{Improve (\%)} & +161.44 & +83.40 & +3.14 & -- \\
\bottomrule
\end{tabular}
\end{table}

% \begin{table}[t]
%     \centering
%     \caption{LLM necessity ablation results on F1 score (\%). Improve reports the average relative gain of LLM Rules over each heuristic across five datasets.}
%     \label{tab:llm_necessity}
%     \resizebox{\columnwidth}{!}{
%     \begin{tabular}{l|cccccc}
%         \toprule
%         \textbf{Metric} & \textbf{Facebook} & \textbf{Amazon} & \textbf{LiveJournal} & \textbf{DBLP} & \textbf{Twitter} & \textbf{Improve} \\
%         \midrule
%         Density & 34.19 & 71.66 & 49.39 & 54.69 & 29.73 & +13.4\% \\
%         Modularity & 38.86 & 85.54 & 52.50 & 52.34 & 31.17 & +5.0\% \\
%         LLM Rules & 40.82 & 92.11 & 54.25 & 54.73 & 32.51 & -- \\
%         \bottomrule
%     \end{tabular}
%     }
%     \vspace{3 pt}
% \end{table}
\begin{table}[t]
    \centering
    \caption{LLM necessity ablation results on F1 score (\%), where LJ denotes LiveJournal. Improv. reports the average relative gain of LLM Rules over each heuristic across five datasets.}
    \label{tab:llm_necessity}
    \small
    \setlength{\tabcolsep}{2pt}
    \begin{tabular}{@{}l|cccccc@{}}
        \toprule
        \textbf{Metric}
        & \textbf{Facebook}
        & \textbf{Amazon}
        & \textbf{LJ}
        & \textbf{DBLP}
        & \textbf{Twitter}
        & \textbf{Improv.} \\
        \midrule
        Density
        & 34.19 & 71.66 & 49.39 & 54.69 & 29.73 & +13.4\% \\
        Modularity
        & 38.86 & 85.54 & 52.50 & 52.34 & 31.17 & +5.0\% \\
        LLM Rules
        & 40.82 & 92.11 & 54.25 & 54.73 & 32.51 & -- \\
        \bottomrule
    \end{tabular}
    \vspace{2 pt}
\end{table}
% \begin{table}[t]
%     \centering
%     \caption{LLM necessity ablation results on F1 score (\%). The best result is highlighted in bold, and the stronger heuristic is underlined.}
%     \label{tab:llm_necessity}
%     \begin{tabular}{l|cccc}
%         \toprule
%         \textbf{Dataset} & Density & Modularity & \textbf{LLM Rules} & \textbf{Improv.} \\
%         \midrule
%         Facebook    & 34.19 & 38.86 & 40.82 & +5.0\% \\
%         Amazon      & 71.66 & 85.54 & 92.11 & +7.7\% \\
%         Livejournal & 49.39 & 52.50 & 54.25 & +3.3\% \\
%         DBLP        & 54.69 & 52.34 & 54.73 & +0.1\% \\
%         Twitter     & 29.73 & 31.17 & 32.51 & +4.3\% \\
%         \bottomrule
%     \end{tabular}
% \end{table}
\subsection{Ablation studies (RQ2)}
\label{sec:ablation}
%首先kego是指直接用特定数量的kego计算指标。w/o merge是指完整流程中去除merge步骤，直接用kego进行refine，再进行rank得到最终的结果。w/o refine指完整流程中去除refine步骤，直接用merge得到的社区进行rank后计算指标。而w/o rank则是指kego进行merge再进行refine,最后不进行rank而是随机选择特定数量的社区计算指标
\noindent \textbf{Stages in LUCID ablation study.} 
Figure~\ref{fig:ablation} presents the stage-wise ablation results of \ModelName. Overall, removing any stage generally degrades performance, demonstrating that the stages play complementary roles in the framework. 
Removing the merge stage causes substantial declines on Amazon and Livejournal but yields a slight improvement on DBLP, suggesting that merge effectiveness is dataset-dependent and sensitive to structure. \textcolor{revisionpurple}{DBLP communities are small and tight with low overlap, merging structurally close $k$-ego subgraphs can blur their boundaries.} Removing refinement consistently impairs performance across multiple datasets, confirming its importance in reducing boundary noise. The removal of selection leads to severe drops on Livejournal and DBLP and moderate declines elsewhere, highlighting the role of global structural filtering in balancing internal cohesion and boundary clarity, especially in networks with substantial community overlap. Although the contribution of each stage varies across datasets, their combined use enables \ModelName to achieve the most robust overall performance.

\noindent \textbf{Ranking metric ablation study.}
% For a detailed breakdown, please refer to Appendix \ref{sec:rank_ablation}.
% Table~\ref{tab:ablation_rank} presents the ablation study comparing different rank algorithms. Our proposed RDC  algorithm consistently outperforms baseline ranking methods across most datasets. Specifically, RDC achieves substantial improvements over Density (+161.44\%) and Modularity (+83.40\%) on average, demonstrating its effectiveness in capturing community quality. Compared to Local Conductance (LC), RDC shows a modest improvement, indicating that both methods effectively leverage conductance-based metrics. Notably, RDC achieves the best performance on Amazon and Livejournal. On Facebook, RDC performs comparably to LC, suggesting that for networks with relatively clear community boundaries, both conductance-based methods are effective. The superior performance of RDC can be attributed to its joint consideration of density and conductance, which enables better identification of communities with both high internal cohesion and clear boundaries across diverse network structures.
Table~\ref{tab:ablation_rank} presents the ablation study comparing different ranking metrics. Our proposed RDC consistently outperforms baselines on most datasets, achieving substantial average gains over Density (+161.44\%) and Modularity (+83.40\%). Compared to Local Conductance (LC), RDC shows modest improvements, indicating that both effectively leverage conductance-based metrics. RDC attains the best performance on Amazon and Livejournal, and performs comparably to LC on Facebook, where community boundaries are relatively clear. The superior performance of RDC can be attributed to its joint consideration of density and conductance, which enables better identification of communities with both high internal cohesion and clear boundaries across diverse network structures.

\begin{figure}[t!]
    \centering
    \includegraphics[width=0.45\textwidth]{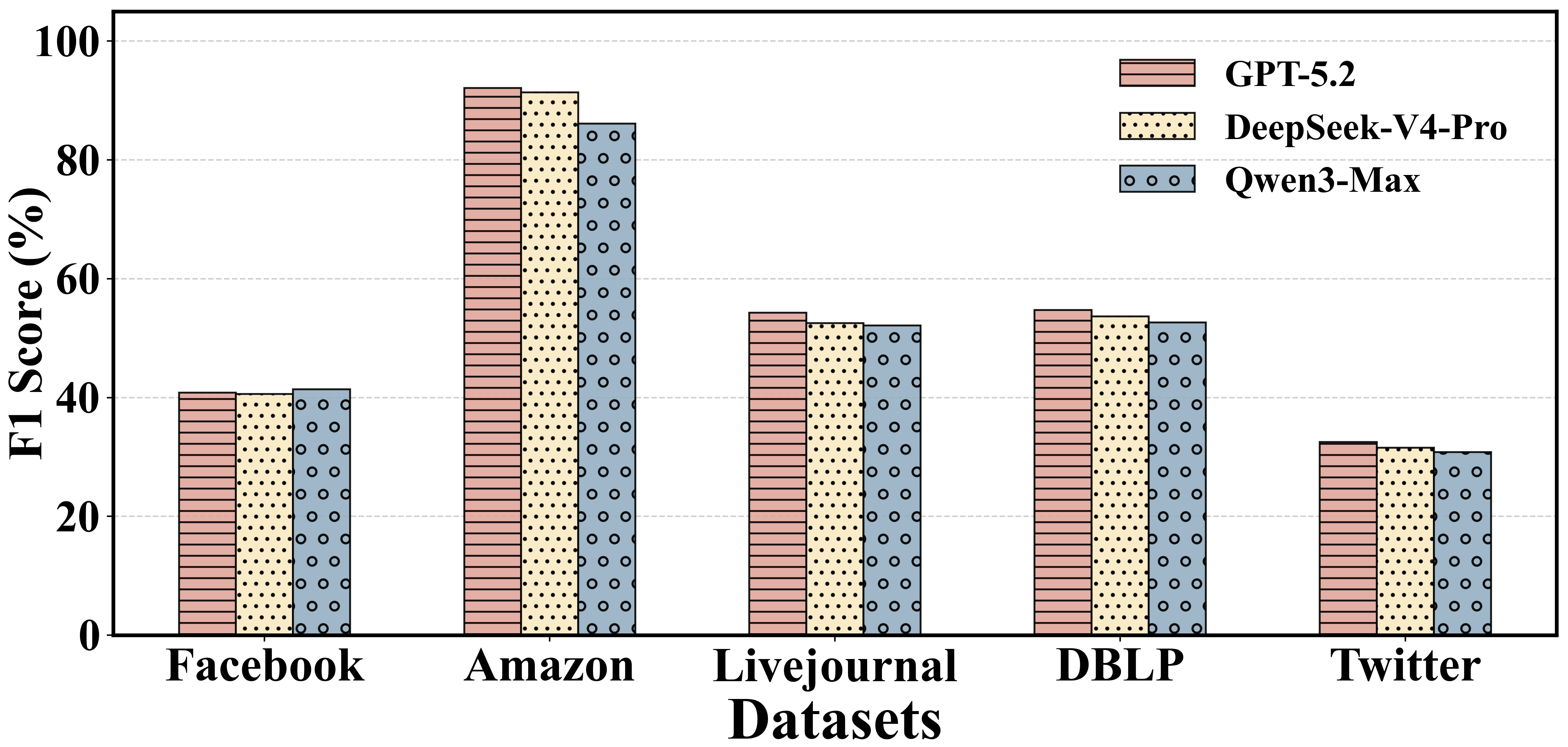} 
    \caption{Performance comparison with different LLMs, measured by F1 score (\%).}
    \label{fig:llm_comparison}
\end{figure}

\noindent \textbf{LLM necessity ablation study.} 
\textcolor{revisionpurple}{To examine whether the LLM is necessary, we remove it from the two stages. We consider Density and Modularity to replace LLM-induced rules. During the merge stage, two communities are merged only if the corresponding metric increases after merging. During the refinement satge, a node is removed only if its removal increases that metric. As shown in Table~\ref{tab:llm_necessity}, LLM-induced rules outperform Density and Modularity by 13.4\% and 5.0\% on average across the five datasets, respectively.
%, supporting our motivation for using the LLM as a rule inducer. 
Further process-level comparisons are provided in Appendix~\ref{sec:process_quality_appendix}.}

\begin{figure}[t!]
    \centering
    \includegraphics[width=\linewidth]{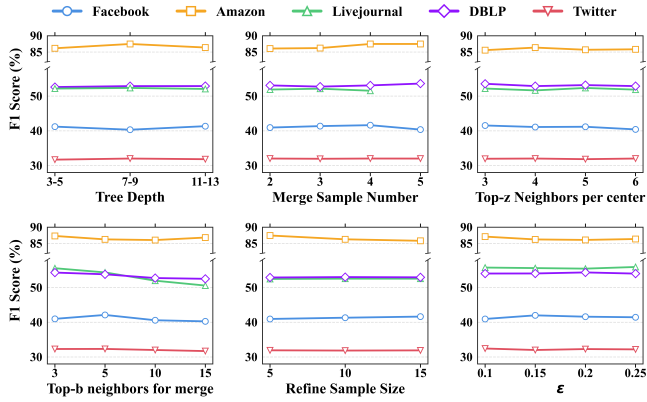}
    
    \caption{Sensitivity analysis of hyperparameters.}
    
    \label{fig:sensitivity}
\end{figure}

\subsection{Efficiency study (RQ3)}
\label{sec:efficiency}
\ModelName achieves competitive computational efficiency while maintaining superior detection performance. Additionally, \ModelName keeps token costs low by limiting each experiment to just two API calls. Detailed efficiency analysis including runtime comparisons with baselines, time with stages, and token cost breakdown are provided in Table ~\ref{tab:efficiency} in Appendix~\ref{sec:efficiency_appendix} and Table~\ref{tab:token_cost} in Appendix~\ref{sec:token_cost_appendix}.

\subsection{Hyperparameter sensitivity (RQ4)}
\label{sec:hyper}
\textcolor{revisionpurple}{We evaluate the sensitivity of \ModelName to LLM selection, tree depth, merge sampling, top-$b$ neighbors, refinement sampling, and the penalty coefficient $\epsilon$. Figures~\ref{fig:llm_comparison} and~\ref{fig:sensitivity} show that performance variations are typically within 1--2\%, demonstrating robustness without extensive tuning. Detailed experimental results and analysis for each hyperparameter are provided in Appendix~\ref{sec:hyper_appendix}.
Based on our analysis, we adopt the default settings in Table~\ref{tab:hyper-parameters} in Appendix~\ref{sec:a_default_hyperparameter}.}

\begin{figure}[t!]
    \centering
    \includegraphics[width=0.5\textwidth]{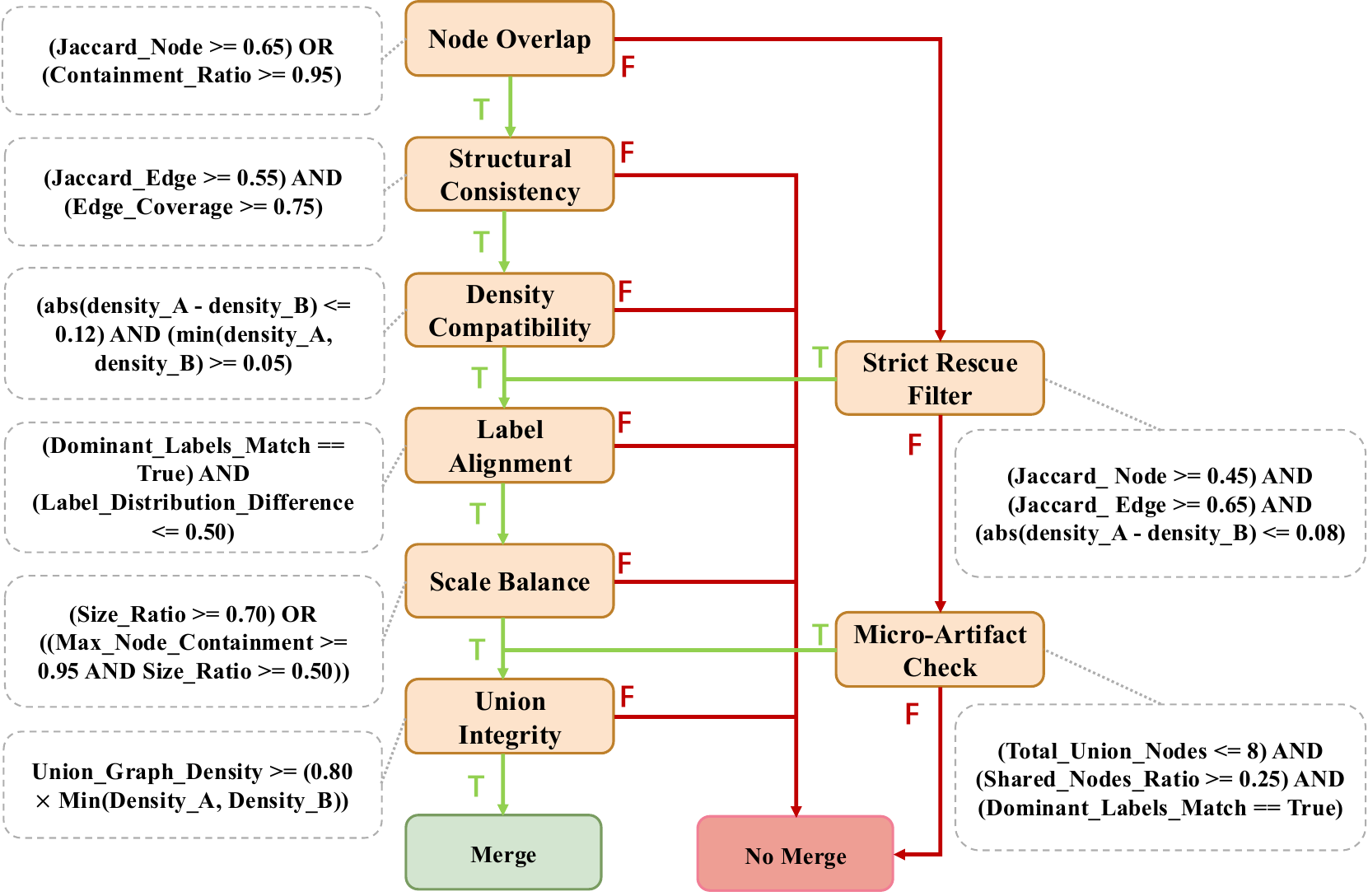} 
    \caption{The decision tree generated by the LLM for subgraph merging.}
    \label{fig:decision_tree}
\end{figure}
\vspace{-2mm}

\begin{figure}[t!]
\centering
\includegraphics[width=0.5\textwidth]{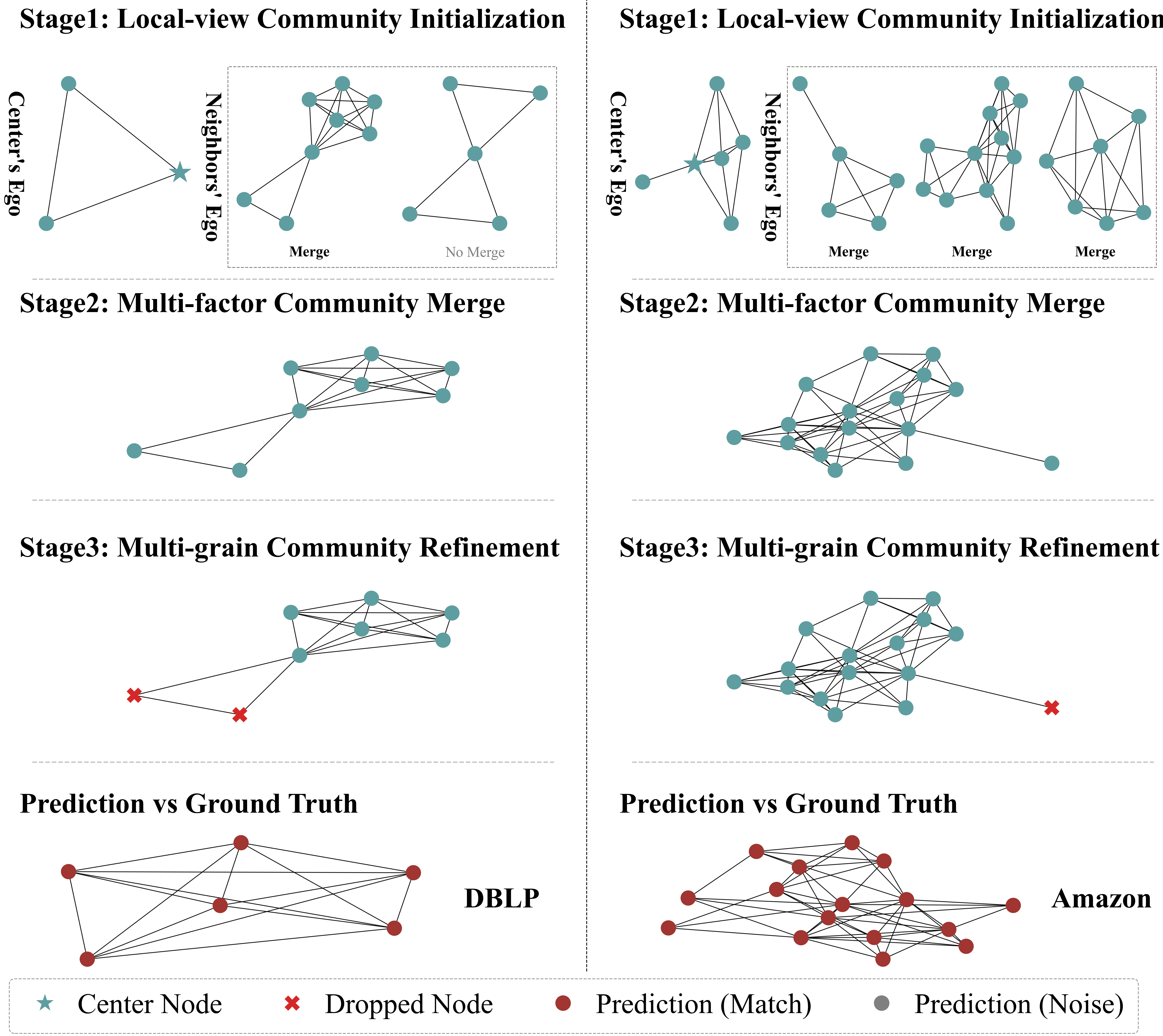}
\caption{Visualization of overall procedure interpretability.}
\label{fig: process}
\end{figure}

\subsection{Case studies (RQ5)}
\label{sec:case study}
\noindent
\textbf{A case study on merging rules.}
% We visualize the LLM-induced merging rules in Figure~\ref{fig:decision_tree}. The process follows a hierarchical decision structure that balances structural connectivity with community quality. 
Figure~\ref{fig:decision_tree} illustrates LLM-induced merging rules in \ModelName, which follows a hierarchical decision structure to balance structural connectivity and community quality. 
(1) Initial filtering: the process begins with the \textit{node overlap gate}, prioritizing subgraph pairs with significant overlap or directional containment. 
(2) Validation sequence: candidates passing the initial gate undergo a multi-step validation, starting with structural consistency and \textit{density compatibility}. Subsequently, \textit{label alignment} and \textit{scale balance} ensure structural uniformity and compatibility. 
(3) Rescue mechanisms: to handle boundary cases, the model incorporates a \textit{strict rescue filter} to recover candidates with moderate overlap and a \textit{micro-artifact check} for small, high-density cliques. 
(4) Final integrity: a \textit{union integrity test} mandates that the merged union retains high density to preserve internal cohesion. 

\noindent
\textbf{A case study on the interpretability of \ModelName.}
% Interpretability means understanding how inputs lead to outputs. 
Interpretability means a transparent decision process from inputs to outputs.
Figure~\ref{fig: process} visualizes \ModelName's LLM-symbolized pipeline on DBLP and Amazon. 
On DBLP, stage 1 shows a center vertex connected to two neighbors (top and bottom) with their ego networks displayed. 
\ModelName merges the top neighbor into the community because it demonstrates strong structural consensus, with high overlap ratio, label consistency, and density similarity validated through dedicated gates. 
Conversely, the bottom neighbor is rejected due to insufficient shared context. The community is then refined by removing nodes with low internal degrees and weak triadic closures, resulting in an exact 5-clique that matches the ground truth. 
% \ModelName merges the top neighbor into the community containing the center vertex through gates for high overlap ratio, label consistency, and density similarity, while rejecting the bottom one for insufficient shared context. Then the communities are refined by removing nodes with low internal degree and weak triadic closure to form a precise 5-clique matching ground truth. 
Similarly, on Amazon, LUCID aggregates dense neighbors and prunes outliers to produce a prediction aligned with ground truth. These examples highlight the transparency and interpretability of \ModelName via explicit structural rules. 
    \section{Conclusion}
    \label{sec.conclusion}
    We propose \ModelName, a four-stage LLM-symbolized framework for interpretable, training-free, and label-free community detection. It constructs communities via local-view initialization, multi-factor merge, multi-grain refinement, and global selection. Rather than acting as end-to-end predictors, the LLM serve as rule inducers: they build a multi-factor decision tree for merging and induce coarse-to-fine rules for refinement to remove boundary noise. \textcolor{revisionpurple}{Experiments show that \ModelName achieves state-of-the-art performance, with average relative improvements of 20.7\% in F1 and 31.9\% in Jaccard over the strongest unsupervised baselines, while outperforming the strongest semi-supervised baselines by 10.1\% and 16.6\%.} 

    \newpage
    \normalem
    \bibliographystyle{ACM-Reference-Format}
    \bibliography{ref}  

    \vspace{3pt}
    \appendix
    % \onecolumn
    \raggedbottom
    \section*{APPENDIX}
    % \section{Notations}
% \label{sec:notation}
% \begin{table}[htbp]
%     \centering
%     \caption{Summary of Notations}
%     \label{tab: notations}
%     \small
%     \renewcommand{\arraystretch}{1.1}
%     \setlength{\tabcolsep}{6pt}
%     \begin{tabular}{cl}
%         \toprule
%         \textbf{Symbol} & \textbf{Description} \\ \midrule
%         $G = (V, E)$ & Input graph \\
%         %记得加上N(v, G), for ease of presentation, we omit G when the context is clear
%         $\mathcal{N}(v)$ & Neighborhood of node $v$ \\
%         $G_v^{(k)}$ & $k$-ego network of $v$ \\ 
%         % $\mathcal{C}_v$ & Community centered on $v$ \\ 
%         % $\mathcal{I} = \{\mathcal{I}_P, \mathcal{I}_B\}$ & Input context in merge stage \\
%         % $v_c$ & Center node \\
%         % $G_{v_c}^{(k)}$ & $k$-ego network of $v_c$ \\
%         % $G_{n_j}^{(k)}$ & $k$-ego network of neighbor $j$ \\
%         $\mathcal{T}$ & Multi-factor decision tree \\
%         % $V_{\text{ego}}, E_{\text{ego}}$ & Ego-network sets \\
%         % $\mathcal{N}_v^{(K)}$ & Top-$K$ neighbors \\
%         $\mathcal{A}$ & Set of available center nodes \\
%         %$M(v,u)$ & Merge decision \\
%         $\mathcal{C_M}$ & Candidate communities after merge \\
%         $\mathcal{C_R}$ & Candidate communities after refinement \\ 
%         % $\bar{\mathcal{C_R}}$ & Remaining graph in selection stage \\
%         $\mathcal{C}$ & Final candidate communities \\\bottomrule
%     \end{tabular}
% \end{table}
\section{Notations}
\label{sec:notation}
\begin{table}[htbp]
    \centering
    \caption{Summary of Notations}
    \label{tab: notations}
    \small
    \renewcommand{\arraystretch}{1.1}
    \setlength{\tabcolsep}{6pt}
    \begin{tabular}{cl}
        \toprule
        \textbf{Symbol} & \textbf{Description} \\ \midrule
        $G = (V, E)$ & Input graph \\
        \textcolor{revisionpurple}{$M$} & \textcolor{revisionpurple}{Number of ground-truth communities} \\
        \textcolor{revisionpurple}{$N$} & \textcolor{revisionpurple}{Number of predicted communities} \\
        %记得加上N(v, G), for ease of presentation, we omit G when the context is clear
        $\mathcal{N}(v)$ & Neighborhood of node $v$ \\
        $G_v^{(k)}$ & $k$-ego network of $v$ \\ 
        % $\mathcal{C}_v$ & Community centered on $v$ \\ 
        % $\mathcal{I} = \{\mathcal{I}_P, \mathcal{I}_B\}$ & Input context in merge stage \\
        % $v_c$ & Center node \\
        % $G_{v_c}^{(k)}$ & $k$-ego network of $v_c$ \\
        % $G_{n_j}^{(k)}$ & $k$-ego network of neighbor $j$ \\
        $\mathcal{T}$ & Multi-factor decision tree \\
        % $V_{\text{ego}}, E_{\text{ego}}$ & Ego-network sets \\
        % $\mathcal{N}_v^{(K)}$ & Top-$K$ neighbors \\
        $\mathcal{A}$ & Set of available center nodes \\
        %$M(v,u)$ & Merge decision \\
        $\mathcal{C_M}$ & Candidate communities after merge \\
        $\mathcal{C_R}$ & Candidate communities after refinement \\ 
        % $\bar{\mathcal{C_R}}$ & Remaining graph in selection stage \\
        $\mathcal{C}$ & Final candidate communities \\\bottomrule
    \end{tabular}
\end{table}

\section{Complexity analysis}
\label{sec:complexity_appendix}
% We analyze the time complexity of \ModelName across its four stages. Let $G=(V,E)$ be a graph with $|V|=n$ nodes and $|E|=m$ edges, $\bar{d}$ denote the average degree, $k$ be the ego network radius, $b$ be the top-$b$ neighbor sampling parameter in the merge stage, $D$ be the decision tree depth, $|\mathcal{C_M}|$ be the number of candidate communities after merging, $|\mathcal{C_R}|$ be the number of candidate communities after refinement, and $\bar{s}$ be the average community size. Stage 1 is dominated by $k$-ego network decomposition, requiring $O(n \cdot \bar{d}^k)$ time when $k \geq 2$. Stage 2 incurs the primary computational overhead of $O(n \cdot b \cdot D)$ from decision tree evaluation for top-$b$ neighbors. Stage 3 scales as $O(|\mathcal{C_M}| \cdot \bar{s} \cdot R)$ where $R$ is the number of rule layers in the coarse-to-fine hierarchy. Stage 4 requires $O(|\mathcal{C_R}| \cdot \bar{s} \cdot \bar{d})$ time for computing structure-stability-regularized scores. Combining all four stages, the overall time complexity is $O(n \cdot \bar{d}^k + n \cdot b \cdot D + |\mathcal{C_M}| \cdot \bar{s} \cdot R + |\mathcal{C_R}| \cdot \bar{s} \cdot \bar{d})$. Stage 2 incurs the dominating cost $O(n \cdot b \cdot D)$, which scales linearly with the number of nodes.
We analyze the time complexity of \ModelName across its four stages. Let $G=(V,E)$ be a graph with $|V|=n$ nodes and $|E|=m$ edges, $\bar{d}$ denote the average degree, $k$ be the ego network radius, $b$ be the top-$b$ neighbor sampling parameter in the merge stage, $D$ be the decision tree depth, $|\mathcal{C_M}|$ be the number of candidate communities after merging, $|\mathcal{C_R}|$ be the number of candidate communities after refinement, and $\bar{s}$ be the average community size. Stage 1 is dominated by $k$-ego network decomposition, requiring $O(n \cdot \bar{d}^k)$ time when $k \geq 2$. Stage 2 incurs the primary computational overhead of $O(n \cdot b \cdot D)$ from decision tree evaluation for top-$b$ neighbors. Stage 3 scales as $O(|\mathcal{C_M}| \cdot \bar{s} \cdot R)$ where $R$ is the number of rule layers in the coarse-to-fine hierarchy. Stage 4 requires $O(|\mathcal{C_R}| \cdot \bar{s} \cdot \bar{d})$ time for computing structure-stability-regularized scores. Combining all four stages, the overall time complexity is $O(n \cdot \bar{d}^k + n \cdot b \cdot D + |\mathcal{C_M}| \cdot \bar{s} \cdot R + |\mathcal{C_R}| \cdot \bar{s} \cdot \bar{d})$.

The space complexity is dominated by storing the graph structure ($O(n+m)$), pre-computed $k$-ego networks ($O(n \cdot \bar{d}^k)$), and candidate communities ($O(|\mathcal{C_M}| \cdot \bar{s} + |\mathcal{C_R}| \cdot \bar{s})$), yielding overall space complexity $O(n+m + n \cdot \bar{d}^k + |\mathcal{C_M}| \cdot \bar{s} + |\mathcal{C_R}| \cdot \bar{s})$.

\section{Default hyperparameter settings}
\label{sec:a_default_hyperparameter}
\begin{table}[htbp]
\centering
\caption{Hyperparameter in \ModelName}
\label{tab:hyper-parameters}
\renewcommand{\arraystretch}{1.2} % Adjust row height for better appearance
\begin{tabular}{c|c}
\toprule
\textbf{Hyperparameter} & \textbf{Value} \\
\midrule
$k$ & Search from \{1, 2\} \\
Number of node labels $K$ & 4 \\
LLM & GPT-5.2 \\
Merge sample size & 3 \\
Max neighbors per center & 3 \\
Top-$M$ neighbors for merge & 10 \\
Tree depth & 11-13 \\
Refine sample size & 10 \\
$\epsilon$ & 0.2 \\
\bottomrule
\end{tabular}
\end{table}

\section{DeepSeek-V4-Pro results}
\label{sec:deepseek_results}
Table~\ref{tab:deepseek_results} reports the performance of \ModelName when DeepSeek-V4-Pro is used as the LLM backbone. \textcolor{revisionpurple}{Although we access it through the official API, its weights are available under the MIT License and support self-hosting~\cite{deepseekv4}, reducing dependence on API availability.} Improvements are relative to the strongest baseline in the corresponding row of Table~\ref{tab:full_comparison}.

\begin{table}[htbp]
    \centering
    \caption{Performance of \ModelName with DeepSeek-V4-Pro (results in percent $\pm$ standard deviation).}
    \label{tab:deepseek_results}
    \small
    \begin{tabular}{c|lcc}
        \toprule
        \textbf{Metric} & \textbf{Dataset} & \textbf{\ModelName (DeepSeek)} & \textbf{Improv.} \\
        \midrule
        \multirow{5}{*}{F1 Score}
        & Facebook    & $40.60_{\pm 0.63}$ & +6.3\% \\
        & Amazon      & $91.36_{\pm 0.41}$ & +9.5\% \\
        & Livejournal & $52.51_{\pm 0.22}$ & -0.8\% \\
        & DBLP        & $53.65_{\pm 0.44}$ & +6.3\% \\
        & Twitter     & $31.56_{\pm 0.05}$ & +6.8\% \\
        \midrule
        \multirow{5}{*}{Jaccard}
        & Facebook    & $31.44_{\pm 0.69}$ & +10.9\% \\
        & Amazon      & $86.34_{\pm 0.70}$ & +15.3\% \\
        & Livejournal & $45.00_{\pm 0.25}$ & +7.9\% \\
        & DBLP        & $43.72_{\pm 0.62}$ & +11.4\% \\
        & Twitter     & $21.05_{\pm 0.05}$ & +7.9\% \\
        \bottomrule
    \end{tabular}
\end{table}

\section{Process-level quality analysis}
\label{sec:process_quality_appendix}

This section provides a detailed process-level analysis of the LLM-induced rules used in the Merge and Refine stages. Unlike community-level F1, this analysis directly examines whether individual merge and node-removal decisions are consistent with the ground truth. We compare LLM Rules with the Density and Modularity heuristics on five datasets. All results are reported in percentages and averaged across datasets.

\subsection{Merge quality}

To examine whether each individual merge decision is reasonable, we evaluate every candidate pair $(C_i,C_j)$ considered during the Merge stage. Our first criterion determines whether the two candidates correspond to the same ground-truth community. Specifically, we identify the ground-truth community having the largest overlap with each candidate and label the pair as positive if the two matched communities are identical:
\[
\operatorname{GT}(C)=\arg\max_{G\in\mathcal{G}}|C\cap G|, 
\\
y_{\mathrm{equal}}(C_i,C_j)
=\mathbf{1}\!\left[\operatorname{GT}(C_i)=\operatorname{GT}(C_j)\right],
\]
where $\mathcal{G}$ denotes the set of ground-truth communities. Based on the predicted merge decisions and these labels, we compute Pair Precision, Pair Recall, and Pair F1.

We additionally evaluate whether merging the two candidates preserves sufficient ground-truth consistency. Let $U=C_i\cup C_j$ and $G^*$ be the ground-truth community with the largest overlap with $U$. The pair is labeled as positive when at least a fraction $\tau$ of the merged community belongs to $G^*$:
\[
U=C_i\cup C_j,
G^*=\arg\max_{G\in\mathcal{G}}|U\cap G|,
\\
y_{\tau}(C_i,C_j)
=\mathbf{1}\!\left[\frac{|U\cap G^*|}{|U|}\geq\tau\right].
\]
We evaluate $\tau\in\{0.6,0.7,0.8\}$ using the same pair-level metrics and adopt $\tau=0.8$ as the strict setting, whose Pair F1 is reported as Strict Pair F1.

As shown in Table~\ref{tab:merge_process_quality}, LLM Rules achieve the best results across all four metrics. Compared with the strongest heuristic, they increase Pair F1 from 32.30\% to 39.64\% and Strict Pair F1 from 16.62\% to 23.57\%. 

\begin{table}[htbp]
\centering
\caption{Merge decision quality (\%) averaged across five datasets. Strict Pair F1 uses the union-coverage threshold $\tau=0.8$.}
\label{tab:merge_process_quality}
\small
\resizebox{\columnwidth}{!}{%
\begin{tabular}{lcccc}
\toprule
\textbf{Rule} &
\textbf{Pair Precision} &
\textbf{Pair Recall} &
\textbf{Pair F1} &
\textbf{Strict Pair F1} \\
\midrule
\textbf{LLM Rules} & \textbf{44.48} & \textbf{62.37} & \textbf{39.64} & \textbf{23.57} \\
Modularity & 42.16 & 48.60 & 32.30 & 16.62 \\
Density    & 37.72 & 0.02  & 0.07  & 0.06  \\
\bottomrule
\end{tabular}%
}
\end{table}

\subsection{Refine quality}

In the Refine stage, each candidate community is first matched to the ground-truth community with the highest Jaccard similarity. Nodes outside the matched community are regarded as noisy nodes, whereas the remaining nodes are treated as true members. We evaluate the refinement rules using Deletion Precision and Deletion Recall, which measure the correctness and coverage of noise removal, respectively. We further report True-member Damage, Purity Gain, and Recall Drop to assess the trade-off between removing noise and preserving valid community members.

\begin{table*}[htbp]
\centering
\caption{Refinement quality (\%) averaged across five datasets. Upward and downward arrows indicate whether higher or lower values are preferred.}
\label{tab:refine_process_quality}
\small
\setlength{\tabcolsep}{8pt}
\begin{tabular}{lccccc}
\toprule
\textbf{Rule} &
\textbf{Deletion Precision $\uparrow$} &
\textbf{Deletion Recall $\uparrow$} &
\textbf{True-member Damage $\downarrow$} &
\textbf{Purity Gain $\uparrow$} &
\textbf{Recall Drop $\downarrow$} \\
\midrule
LLM Rules  & 83.11 & 23.18 & 5.02  & 3.27 & 1.66  \\
Modularity & 99.80 & 0.06  & 0.00  & 0.01 & 0.00  \\
Density    & 72.17 & 60.03 & 31.80 & 9.28 & 12.55 \\
\bottomrule
\end{tabular}
\end{table*}

As shown in Table~\ref{tab:refine_process_quality}, Modularity is overly conservative. Although it causes almost no damage to true members, its Deletion Recall is only 0.06\% and its Purity Gain is nearly zero, indicating that it removes little noise. Density exhibits the opposite behavior: it achieves higher Deletion Recall and Purity Gain but removes substantially more true members, resulting in 6.3$\times$ higher True-member Damage and 7.6$\times$ higher Recall Drop than LLM Rules. In comparison, LLM Rules remove a meaningful proportion of noisy nodes while limiting the loss of true members. This result indicates that LLM-induced refinement provides a better balance between community denoising and member preservation, avoiding both the ineffective refinement of Modularity and the excessive pruning of Density.

\section{Efficiency analysis}
\label{sec:efficiency_appendix}
As shown in Table~\ref{tab:efficiency}, \ModelName strikes a balance between computational efficiency and detection performance. Compared to deep learning baselines with convergence difficulties or high time costs, \ModelName maintains better robustness. Notably, on large datasets where deep learning methods fail to converge (ComE and CmtyGAN marked as ``NA'' in Table~\ref{tab:efficiency}), \ModelName successfully processes all datasets. On datasets where comparisons are available, \ModelName achieves significant speedups: 52$\times$ faster than SEAL on Facebook, 472$\times$ faster than ComE on Amazon, and 213$\times$ faster than CmtyGAN on Livejournal.

Among all stages, the merge stage incurs the dominating computational costs, consuming 78.9\%--93.6\% of total runtime across datasets, due to deeper decision tree generation for enhanced interpretability. Meanwhile, the \textit{Refine} stage is highly efficient and only accounts for 3.4\%--18.8\% of total runtime. The \textit{Selection} stage introduces minimal overhead, consuming less than 1.5\% of total runtime. The running times reported in Table~\ref{tab:efficiency} include LLM API call latency but exclude the subsequent evaluation process.

% [Before revision] (No explicit statement on number of LLM calls or large-scale bottleneck.)
\textcolor{revisionpurple}{Across all datasets, \ModelName uses only two calls to the LLM API, one in the \textit{Merge} stage and one in the \textit{Refine} stage, independent of graph size.}
\textcolor{revisionpurple}{The main runtime bottleneck is the \textit{Merge} stage, while token usage remains moderate due to top-$z$/top-$b$ bounded context construction.}
\textcolor{revisionpurple}{We further evaluate \ModelName on \textit{wiki-topcats} (1.79M nodes, 28.5M edges), where total runtime is 6,095.5s with only $\sim$20k LLM tokens. This result indicates that large-graph runtime is dominated by graph-side processing rather than LLM context growth.}
\textcolor{revisionpurple}{To further clarify the cost of Stage 1, Table~\ref{tab:kego_runtime} reports the runtime of $k$-ego initialization across datasets.}

\begin{table}[htbp]
\centering
\caption{\textcolor{revisionpurple}{Runtime of $k$-ego initialization.}}
\label{tab:kego_runtime}
\small
\begin{tabular}{lcc}
\toprule
\textbf{Dataset} & \textbf{$k$} & \textbf{Compute $k$-ego (s)} \\
\midrule
Facebook & \textcolor{revisionpurple}{1} & \textcolor{revisionpurple}{0.05} \\
Amazon & \textcolor{revisionpurple}{2} & \textcolor{revisionpurple}{15.09} \\
DBLP & \textcolor{revisionpurple}{1} & \textcolor{revisionpurple}{0.77} \\
LiveJournal & \textcolor{revisionpurple}{1} & \textcolor{revisionpurple}{0.76} \\
Twitter & \textcolor{revisionpurple}{1} & \textcolor{revisionpurple}{0.88} \\
\bottomrule
\end{tabular}
\end{table}

\begin{table*}[htbp]
\centering
\caption{Efficiency study (RQ2) across five datasets, measured by execution time(s). The values in parentheses represent the execution time of the four stages within our \textsc{\ModelName} framework (local-view community initialization, multi-factor community merge, multi-grain community refinement, global-view community selection), including the time consumed by LLM API calls.} 
\label{tab:efficiency}
\renewcommand{\arraystretch}{1.2} 
\setlength{\tabcolsep}{5pt} 
\begin{tabular}{l | ccc ccc c | c}
\toprule
 \textbf{Dataset} & \textbf{BigClam} & \textbf{ComE} & \textbf{CmtyGAN} & \textbf{Bespoke} & \textbf{SEAL} & \textbf{CLARE} & \textbf{ProCom} & \textbf{\ModelName (Ours)} \\ 
\midrule
Facebook    & 2  & 611         & 1,620         & 4   & 10,440 & 135 & 10  & 201 (2/178/20/1)\\
Amazon      & 2  & 62,760    & 23,460   & 207 & 4,080 & 118 & 206 & 133 (2/105/25/1)\\
Livejournal & 76 & NA          & 271,080        & 67  & 14,580 & 784 & 145 & 1,275 (22/1,194/43/14)\\
DBLP        & 34 & NA   & NA      & 730 & 2,520    & 400 & 305 & 299 (18/242/35/4)\\
Twitter     & 46 & 273,120         & NA   & 538 & 2,280    & 1,920  & 212 & 1,317 (27/1,144/128/18)\\
\bottomrule
\end{tabular}
\end{table*}

\section{Token cost of five datasets}
\label{sec:token_cost_appendix}
LUCID involves two API calls that occur in the Merge and Refine stages, respectively.
In the \textbf{Merge} stage, the input incorporates the dataset description and sampled $k$-ego networks with detailed structural context. Consequently, input token consumption is relatively high and varies with graph density. The LLM output consists of a structured decision tree and an executable Python function for merging.
In the \textbf{Refine} stage, the input comprises the dataset description and a small batch of candidate communities. The LLM generates a node-level refinement function as output to optimize community boundaries. Due to the compact context, token usage in this stage remains low and consistent across datasets.
By limiting the interaction to these two specific API calls, LUCID ensures low token costs, as detailed in Table \ref{tab:token_cost}.
% \begin{table}[htbp]
% \centering
% \caption{Token usage and cost across different datasets. Costs are estimated based on OpenAI GPT-5.2 standard pricing (\$1.75/1M input tokens and \$14.00/1M output tokens).}
% \label{tab:token_cost}
% \resizebox{\linewidth}{!}{%
% \begin{tabular}{lccc}
% \toprule
% \textbf{Dataset} & \textbf{Merge Stage Tokens} & \textbf{Refine Stage Tokens} & \textbf{Est. Cost (\$)} \\
% & (Input / Output) & (Input / Output) & (Merge / Refine / Total) \\
% \midrule
% Amazon & $\sim$12k / 5k & $\sim$1k / 2k & $\sim$0.09 / 0.03 / 0.12 \\
% Facebook & $\sim$90k / 6k & $\sim$1k / 2k & $\sim$0.24 / 0.03 / 0.27 \\
% DBLP & $\sim$85k / 5k & $\sim$1k / 2k & $\sim$0.22 / 0.03 / 0.25 \\
% LiveJournal & $\sim$50k / 5k & $\sim$1k / 2k & $\sim$0.16 / 0.03 / 0.19 \\
% Twitter & $\sim$120k / 6k & $\sim$1k / 2k & $\sim$0.29 / 0.03 / 0.32 \\
% \bottomrule
% \end{tabular}%
% }
% \vspace{6pt}
% \end{table}
\begin{table*}[h!]
\centering
\caption{Token usage and cost across different datasets. Costs are estimated based on OpenAI GPT-5.2 standard pricing (\$1.75/1M input tokens and \$14.00/1M output tokens).}
\label{tab:token_cost}
\small
\setlength{\tabcolsep}{7pt}
\begin{tabular}{lccc}
\toprule
\textbf{Dataset} & \textbf{Merge Stage Tokens} & \textbf{Refine Stage Tokens} & \textbf{Est. Cost (\$)} \\
& (Input / Output) & (Input / Output) & (Merge / Refine / Total) \\
\midrule
Amazon & $\sim$12k / 5k & $\sim$1k / 2k & $\sim$0.09 / 0.03 / 0.12 \\
Facebook & $\sim$90k / 6k & $\sim$1k / 2k & $\sim$0.24 / 0.03 / 0.27 \\
DBLP & $\sim$85k / 5k & $\sim$1k / 2k & $\sim$0.22 / 0.03 / 0.25 \\
LiveJournal & $\sim$50k / 5k & $\sim$1k / 2k & $\sim$0.16 / 0.03 / 0.19 \\
Twitter & $\sim$120k / 6k & $\sim$1k / 2k & $\sim$0.29 / 0.03 / 0.32 \\
\bottomrule
\end{tabular}%
\vspace{6pt}
\end{table*}

\section{The stability of \ModelName}
\label{sec:stability_appendix}
% [Before revision] (No dedicated section for randomness / robustness.)
\textcolor{revisionpurple}{Unlike the fixed-prompt trials in the main experiments, here we vary the sampling seed and report mean$\pm$std over three runs.}
\textcolor{revisionpurple}{The result indicates moderate variance across sampling seeds, while maintaining competitive performance.}

\begin{table}[htbp]
\centering
\caption{\textcolor{revisionpurple}{Stability under sampling-seed variance (3 seeds).}}
\label{tab:stability_sampling_seed}
\small
\setlength{\tabcolsep}{6pt}
\begin{tabular}{lcc}
\toprule
\textbf{Dataset} & \textbf{F1 (mean$\pm$std)} & \textbf{Jaccard (mean$\pm$std)} \\
\midrule
Facebook & \textcolor{revisionpurple}{$38.52 \pm 0.97$} & \textcolor{revisionpurple}{$28.50 \pm 0.90$} \\
Amazon & \textcolor{revisionpurple}{$91.30 \pm 1.45$} & \textcolor{revisionpurple}{$86.03 \pm 2.34$} \\
LiveJournal & \textcolor{revisionpurple}{$54.34 \pm 1.86$} & \textcolor{revisionpurple}{$46.80 \pm 1.82$} \\
\bottomrule
\end{tabular}
\end{table}

\section{Robustness to large-community fragmentation}
\label{sec:large_community_fragmentation}
\textcolor{revisionpurple}{To assess whether small-$k$ ego decomposition fragments large communities, we isolate ground-truth communities with diameter no smaller than 3. As shown in Table~\ref{tab:large_community_kego}, performance on these communities remains comparable to full-dataset performance.}

\begin{table}[h!]
\centering
\caption{\textcolor{revisionpurple}{Large-community analysis under small $k$.}}
\label{tab:large_community_kego}
\small
\begin{tabular}{lcccc}
\toprule
\textbf{Dataset} & \textbf{GT($d\geq3$)/Total} & \textbf{$k$} & \textbf{F1($d\geq3$)} & \textbf{F1(full)} \\
\midrule
Twitter & \textcolor{revisionpurple}{14.20\%} & \textcolor{revisionpurple}{1} & \textcolor{revisionpurple}{34.90} & \textcolor{revisionpurple}{32.51} \\
Facebook & \textcolor{revisionpurple}{21.54\%} & \textcolor{revisionpurple}{1} & \textcolor{revisionpurple}{38.86} & \textcolor{revisionpurple}{40.82} \\
\bottomrule
\end{tabular}
\end{table}

\section{Robustness to node-role label}
\textcolor{revisionpurple}{We also evaluate robustness to node-role noise by randomly perturbing 25\%, 50\%, and 75\% of role assignments. As shown in Table~\ref{tab:role_noise_robustness}, performance remains stable because node roles serve as auxiliary structural signals rather than hard constraints.}
\begin{table}[t]
    \centering
    \caption{Performance under different noise ratios with F1.}
    \label{tab:noise_ratio}
    \begin{tabular}{lcccc}
        \toprule
        \textbf{Dataset} & \textbf{0\%} & \textbf{25\%} &
        \textbf{50\%} & \textbf{75\%} \\
        \midrule
        Facebook    & 40.95 & 40.05 & 41.06 & 40.51 \\
        Twitter     & 33.03 & 28.81 & 32.00 & 32.76 \\
        LiveJournal & 56.57 & 50.04 & 47.55 & 52.11 \\
        \bottomrule
    \end{tabular}
\end{table}

\section{Hyperparameter sensitivity analysis}
\label{sec:hyper_appendix}
% \begin{table}[t!]
% \centering
% \caption{Default hyperparameter settings used in \ModelName.}
% \label{tab:default_hyperparameters}
% \small
% \begin{tabular}{lc}
% \toprule
% \textbf{Symbol} & \textbf{Value} \\
% \midrule
% $\mathcal{M}_{\mathrm{LLM}}$ & GPT-5.2 \\
% $[d_{\min}, d_{\max}]$ & $[11, 13]$ \\
% $s_m$ & $3$ \\
% $z$ & $3$ \\
% $b$ & $3$ \\
% $s_r$ & $10$ \\
% $\epsilon$ & $0.2$ \\
% \bottomrule
% \end{tabular}
% \end{table}

This section provides detailed sensitivity analysis for all hyperparameters evaluated in our framework. We systematically examine the impact of different hyperparameter settings on model performance across multiple datasets.

\subsection{LLM sensitivity}
As shown in Figure~\ref{fig:llm_comparison}, \ModelName maintains stable performance across three different LLMs, with overall performance differences generally within 1.5\% across four mainstream datasets, indicating good adaptability to different LLM architectures. On datasets with clearer structural patterns (Facebook and Amazon), DeepSeek-V4-Pro achieves slightly better results, while GPT-5.2 demonstrates superior performance on more challenging datasets (Livejournal and DBLP). However, DeepSeek-V4-Pro and Qwen3-Max exhibit limited context windows and output instability, often failing to generate fixed-format responses; on the Twitter dataset, these models frequently exceed context window limits under default settings. Therefore, we select GPT-5.2 as our default LLM for its superior stability and larger context capacity.
% [Before revision] (Twitter-specific cross-model results were not explicitly reported.)
\textcolor{revisionpurple}{On Twitter, GPT-5.2, Qwen3-Max, and DeepSeek-V4-Pro achieve F1 scores of 0.320, 0.308, and 0.320, respectively, as shown in Table~\ref{tab:llm_generalization_twitter}. These results are consistent with the sensitivity analysis and indicate stable cross-model behavior under bounded-context prompting.}

% \begin{table}[htbp]
% \centering
% \caption{\textcolor{revisionpurple}{Twitter performance across LLM backbones (F1).}}
% \label{tab:llm_generalization_twitter}
% \small
% \begin{tabular}{lc}
% \toprule
% \textbf{LLM} & \textbf{F1} \\
% \midrule
% GPT-5.2 & \textcolor{revisionpurple}{0.320} \\
% Qwen3-Max & \textcolor{revisionpurple}{0.308} \\
% DeepSeek-V3.2 & \textcolor{revisionpurple}{0.320} \\
% \bottomrule
% \end{tabular}
% \end{table}

\subsection{Tree depth sensitivity}
Figure~\ref{fig:sensitivity} shows that performance variations across different depth settings are marginal, indicating that the framework is robust to tree depth choices. We observe that deeper trees (11--13) capture more structural information and provide richer interpretability, while maintaining comparable performance. Therefore, we adopt the depth range of 11--13 as our default setting. Using a range rather than a fixed value allows the LLM flexibility in generating decision trees, as the LLM

may deviate from exact specifications even when given fixed depth constraints; providing an acceptable range encourages more consistent adherence to the depth requirements.

\subsection{Merge sample size sensitivity}
We analyze two sampling hyperparameters: the number of sampled centers and the top-$z$ neighbors per center. As shown in Figure~\ref{fig:sensitivity}, \ModelName exhibits low sensitivity to both parameters. For the number of sampled centers, values of 3 and 4 yield comparable performance; we select 3 as the default as it represents a stable case that maintains prompt length within the LLM's context window, effectively avoiding \textit{NA} issues on Livejournal dataset. Similarly, using 3 neighbors per center provides stable performance across datasets, with minimal fluctuations. Increasing either parameter occasionally yields slight improvements but risks exceeding context limits. Therefore, we adopt 3 as the default value for both parameters to ensure stable performance while maintaining computational efficiency.

\subsection{Top-$b$ neighbors for merge sensitivity}
The parameter $b$ controls the number of candidate neighbors considered for each center node during the merge process. 
We start with $b=3$, the smallest value with meaning merging decisions. 
As shown in Figure~\ref{fig:sensitivity}, performance remains relatively stable across different $b$ values, with $b=3$ achieving the best results on most datasets. 
As $b$ increases, performance gradually degrades, particularly on Livejournal and DBLP. This implies that considering too many neighbors may introduce noise or weakly related candidates. 
We attribute this optimal performance at $b=3$ to the alignment between decision tree construction and execution stages: during prompt construction, we sample top-3 neighbors for LLM decision tree generation, and the LLM learns rules specifically tailored for this three-neighbor structure. Therefore, executing merge decisions with $b=3$ matches the learned patterns and leads to superior performance. Based on these observations, we select $b=3$ as the default value.

\subsection{Refine sample size sensitivity}
Figure~\ref{fig:sensitivity} shows that performance is robust to changes in refine sample size across different datasets. While Facebook and Livejournal show slight improvements with more samples, Amazon reaches its peak at size 5, suggesting that extra samples may introduce redundancy. Given the overall stable performance across different sample sizes, we select 10 as the default value.

\subsection{$\epsilon$ (penalty coefficient) sensitivity}
As shown in Figure~\ref{fig:sensitivity}, performance remains stable when $\epsilon$ varies within 0.1–0.25 across different datasets. 
Values of $\epsilon = 0.15$ and $0.2$ achieve consistently better or comparable results, suggesting that a moderate penalty on chain-like communities is beneficial. Given the overall stable performance across this range, we select $\epsilon = 0.2$ as the default value.
\textcolor{revisionpurple}{For new graphs, we recommend $\epsilon=0.2$ by default, with smaller values for sparse graphs and larger values for dense or noisy graphs.}

\section{Prompt details}
\label{prompt details}
This section presents the prompts used in the Merge and Refine stages. Dataset descriptions and sampled communities are inserted into the corresponding placeholders at runtime.

\begin{tcolorbox}[
    enhanced,
    breakable,
    sharp corners,
    boxrule=0.5pt,
    colback=white,
    colframe=black,
    label={box:prompt_spec},
    title=\textbf{Box 1: Merge-Stage Prompt},
    fonttitle=\bfseries\sffamily,
    coltitle=black,
    attach boxed title to top left={yshift=-2mm, xshift=2mm},
    boxed title style={colback=white, colframe=white},
    drop shadow
]
\small
You are an expert in graph community detection and structural reasoning. Given a dataset description and representative $k$-ego samples, derive conservative and interpretable rules for merging neighboring $k$-ego networks.

\textbf{Input:}
\begin{itemize}
    \item \textit{Dataset description}: global graph statistics and expected community characteristics.
    \item \textit{$k$-ego samples}: a center $k$-ego network and the $k$-ego networks of its direct neighbors, represented by node IDs, node labels, and either node degrees or adjacency lists.
\end{itemize}

\textbf{Task:}
\begin{itemize}
    \item Analyze all samples jointly and construct a decision tree with 8--12 rule nodes. For each center--neighbor pair, decide whether the neighbor network should be merged with the center network.
    \item Define each rule using interpretable structural and semantic evidence, such as node overlap, connectivity, density or compactness, size balance, label agreement, and cohesion after merging. Adapt thresholds to the dataset description and favor precision over recall.
    \item Each rule node must specify a unique ID, a concise metric name, executable Python-style condition code, and its true and false successors. The tree must terminate at either \texttt{merge} or \texttt{no\_merge}.
    \item Convert the learned tree into an equivalent executable function that evaluates all neighboring $k$-ego networks, reuses shared computations, and returns the union of the center network and all neighbors selected for merging.
\end{itemize}

\textbf{Output:} Return only one valid JSON object containing (1) the merge decision tree, (2) several representative decision-path traces, and (3) the executable merge function. The function must be concise, CPU-compatible, and directly executable for the specified input representation.

\textbf{Output schema:}
\begin{lstlisting}[basicstyle=\ttfamily\scriptsize,breaklines=true,columns=fullflexible,showstringspaces=false]
{
  "merge_decision_tree": {
    "nodes": [
      {
        "node_id": "root",
        "kind": "rule",
        "name": "<metric_name>",
        "code": "def rule_fn(meta): ...",
        "true_child": "<next_node>",
        "false_child": "<next_node_or_leaf>"
      },
      {
        "node_id": "leaf_merge",
        "kind": "leaf",
        "action": "merge"
      }
    ],
    "sample_decisions": [
      {
        "sample_id": 0,
        "neighbor_id": "<node_id>",
        "merge": true,
        "path": ["root", "...", "leaf_merge"],
        "merged_nodes": ["<node_ids>"],
        "reason": "<decisive_factors>"
      }
    ]
  },
  "merge_execution_function": {
    "code_degree": "def merge_k_ego_networks(...): ...",
    "code_adjacency_list": "def merge_k_ego_networks(...): ..."
  }
}
\end{lstlisting}
\end{tcolorbox}
\begin{tcolorbox}[
    enhanced,
    sharp corners,
    boxrule=0.5pt,
    colback=white,
    colframe=black,
    label={box:prompt_spec2},
    title=\textbf{Box 2: Refine-Stage Prompt},
    fonttitle=\bfseries\sffamily,
    coltitle=black,
    attach boxed title to top left={yshift=-2mm, xshift=2mm},
    boxed title style={colback=white, colframe=white},
    drop shadow
]
\small
You are an expert in graph community detection. Candidate communities preserve most of the true community structure but may contain noisy boundary nodes. Given a dataset description and representative candidate communities, generate a scalable rule-based procedure for removing such nodes.

\textbf{Input:}
\begin{itemize}
    \item \textit{Dataset description}: global graph and community characteristics.
    \item \textit{Candidate-community samples}: node IDs, either an induced adjacency list or node degrees, and optional node labels. These samples guide rule induction; the resulting procedure is applied to the full collection.
\end{itemize}

\textbf{Task:}
\begin{itemize}
    \item Derive dataset-adaptive refinement criteria by jointly considering structural signals (e.g., degree, leaves, and low-core nodes), topological support (e.g., weak triadic closure), and, when available, label consistency with neighboring or majority nodes.
    \item Generate one batch refinement function that identifies nodes inconsistent with each candidate community while preserving coherent community cores.
    \item Ensure that the function handles either input representation, optional labels, missing fields, and empty communities. It must run on CPU using Python and NumPy with $O(n)$--$O(n\log n)$ complexity and scale to tens of thousands of communities.
\end{itemize}

\textbf{Output:} Return only the directly executable function \texttt{batch\_refine\_communities}. Its return value must be a mapping from each \texttt{community\_index} to the corresponding \texttt{list\_of\_nodes\_to\_drop}. 
\end{tcolorbox}

\section{Algorithm details}
\label{sec:algorithm detals}
In this section, we show the detailed algorithm processes of merge stage and entire \ModelName in Algorithms \ref{alg:merge} and \ref{alg:lucid}, respectively.

\begin{algorithm}[H] 
\caption{Merge Phase Algorithm}
\label{alg:merge}
\begin{algorithmic}[1]
\renewcommand{\algorithmicrequire}{\textbf{Input:}}
\renewcommand{\algorithmicensure}{\textbf{Output:}}

\Require Graph $G = (V, E)$, Input context $\mathcal{I} = \{\mathcal{I}_P, \mathcal{I}_B\}$ (where $\mathcal{I}_B$ contains top-$z$ neighbors for each sample center), sampling parameters $z$ and $b$
\Ensure Candidate communities $\mathcal{C_{M}}$

\State \codecmt{Pre-compute ego networks for all nodes.}
\State \textbf{Pre-compute} $V_{\text{ego}}(v), E_{\text{ego}}(v)$ for all nodes $v \in V$

\State \codecmt{Multi-factor Decision Tree Construction.}
\State \textbf{Generate} decision tree $\mathcal{T}$ via LLM using context $\mathcal{I}$
\State \textbf{Initialize:} $\mathcal{C_M} \gets \{V_{\text{ego}}(v) \mid v \in V\}$, $\mathcal{A} \gets V$

\State \codecmt{Dynamic Merge Orchestration.}
\State \textbf{Compute} $\text{EgoDensity}(v) = \frac{2|E_{\text{ego}}(v)|}{|V_{\text{ego}}(v)|(|V_{\text{ego}}(v)|-1)}$ for all $v \in \mathcal{A}$
\State \textbf{Sort} $\mathcal{A}$ by $\text{EgoDensity}$ in descending order
\For{each center node $v \in \mathcal{A}$}
    \State \textbf{Compute} Jaccard similarity $J(v,u) = \frac{|V_{\text{ego}}(v) \cap V_{\text{ego}}(u)|}{|V_{\text{ego}}(v) \cup V_{\text{ego}}(u)|}$ for all $u \in \mathcal{N}(v)$
    \State \textbf{Select} top-$b$ neighbors $\mathcal{N}_b(v)$ from $\{u \in \mathcal{N}(v) \mid u \in \mathcal{A}\}$ ranked by $J(v,u)$
    \For{each neighbor $u \in \mathcal{N}_b(v)$}
        \State \textbf{Evaluate} merge decision $M(v,u) \gets \mathcal{T}$
        \If{$M(v,u) = 1$}
            \State \textbf{Merge}: $\mathcal{C}_v \gets \mathcal{C}_v \cup \{u\}$
            \State \textbf{Update} available set: $\mathcal{A} \gets \mathcal{A} \setminus V_{\text{ego}}(u)$
        \EndIf
    \EndFor
\EndFor

\State \textbf{Return} $\mathcal{C_{M}}$
\end{algorithmic}
\end{algorithm}

Algorithm~\ref{alg:merge} details the merge stage process. The complete \ModelName framework algorithm is presented in Algorithm~\ref{alg:lucid}.

\begin{algorithm}[H] 
\caption{The LUCID Framework}
\label{alg:lucid}
\begin{algorithmic}[1]
\renewcommand{\algorithmicrequire}{\textbf{Input:}}
\renewcommand{\algorithmicensure}{\textbf{Output:}}

\Require Graph $G = (V, E)$, merge sample number $s_m$, sampling parameters $z$ and $b$, tree depth range $[d_{\min}, d_{\max}]$, refine sample size $s_r$, penalty coefficient $\epsilon$, top-$N$ selection parameter
\Ensure Final communities $\mathcal{C}$

\State \codecmt{Stage 1: Local-view Community Initialization.}
\For{each node $u \in V$}
    \State \textbf{Compute} Jaccard similarity $J(u,v)$ for all $v \in \mathcal{N}(u)$
    \State \textbf{Extract} feature vector $\mathbf{x}_u$ from percentiles of $\{J(u,v)\}_{v \in \mathcal{N}(u)}$
\EndFor
\State \textbf{Assign} role labels $\{l_u\}_{u \in V}$ via K-means clustering ($K=4$) on $\{\mathbf{x}_u\}_{u \in V}$
\State \textbf{Decompose} graph into $k$-ego networks: $\{G_u^{(k)}\}_{u \in V}$

\State \codecmt{Stage 2: Multi-factor Community Merge.}
\State \textbf{Prepare} merge input context: select $s_m$ representative center nodes $\{v_c^{(i)}\}_{i=1}^{s_m}$, compute Jaccard similarity $J(v_c^{(i)},u)$ for all $u \in \mathcal{N}(v_c^{(i)})$, and select top-$z$ neighbors $\mathcal{N}_z(v_c^{(i)})$ for each center to form sample blocks $\{\mathcal{I}_B^{(i)}\}_{i=1}^{s_m}$
\State \textbf{Construct} input context $\mathcal{I} = \{\mathcal{I}_P, \{\mathcal{I}_B^{(i)}\}_{i=1}^{s_m}\}$ with tree depth range $[d_{\min}, d_{\max}]$
\State \codecmt{Call Merge Algorithm (Algorithm~\ref{alg:merge}).}
\State \codecmt{Input: $G$, $\mathcal{I}$, $b$; Output: $\mathcal{C_M}$}
\State $\mathcal{C_M} \gets$ \textsc{Merge}($G$, $\mathcal{I}$, $b$)

\State \codecmt{Stage 3: Multi-grain Community Refinement.}
\State \textbf{Sample} $s_r$ candidate communities $\{\mathcal{C}_j\}_{j=1}^{s_r}$ from $\mathcal{C_M}$ as examples
\State \textbf{Generate} coarse-to-fine rule set $\mathcal{R}$ via LLM using context from sampled communities $\{\mathcal{C}_j\}_{j=1}^{s_r}$
\State \textbf{Initialize:} $\mathcal{C_R} \gets \emptyset$
\For{each candidate community $\mathcal{C} \in \mathcal{C_M}$}
    \State \textbf{Initialize:} $\mathcal{C}' \gets \mathcal{C}$
    \For{each node $w \in \mathcal{C}$}
        \State \textbf{Apply} rules in $\mathcal{R}$ progressively from coarse to fine
        \If{node $w$ violates refinement rules}
            \State \textbf{Remove}: $\mathcal{C}' \gets \mathcal{C}' \setminus \{w\}$
        \EndIf
    \EndFor
    \If{$|\mathcal{C}'| \geq 2$}
        \State \textbf{Add}: $\mathcal{C_R} \gets \mathcal{C_R} \cup \{\mathcal{C}'\}$
    \EndIf
\EndFor

\State \codecmt{Stage 4: Global-view Community Selection.}
\For{each community $\mathcal{C} \in \mathcal{C_R}$}
    \State \textbf{Compute} $RDC(\mathcal{C}) = \frac{Cut(\mathcal{C}, \bar{\mathcal{C}})}{|E(\mathcal{C})| - \epsilon(|\mathcal{C}| - 1)}$
\EndFor
\State \textbf{Sort} $\mathcal{C_R}$ by $RDC$ scores in ascending order
\State \textbf{Select} top-$N$ communities: $\mathcal{C} \gets \{\mathcal{C}_1, \mathcal{C}_2, \ldots, \mathcal{C}_N\}$ from sorted $\mathcal{C_R}$

\State \textbf{Return} $\mathcal{C}$
\end{algorithmic}
\end{algorithm}

\end{document}